\documentclass{article}

\usepackage[final]{neurips_2022}

\usepackage[utf8]{inputenc} 
\usepackage[T1]{fontenc}    
\usepackage{hyperref}       
\usepackage{url}            
\usepackage{booktabs}       
\usepackage{amsfonts}       
\usepackage{nicefrac}       
\usepackage{microtype}      
\usepackage{xcolor}         

\usepackage{amsmath,amsfonts,amssymb}
\usepackage{amsthm}
\usepackage{bm}
\usepackage{enumitem,array,booktabs}
\usepackage{algorithm}
\usepackage{algorithmic}
\usepackage{bbm}
\usepackage{tabulary,multirow}
\usepackage{dsfont}
\usepackage{ulem}

\title{A Theory of PAC Learnability under Transformation Invariances}

\author{%
  Han Shao\\
  Toyota Technological Institute Chicago\\
  Chicago, 60637\\
  \texttt{han@ttic.edu} \\
   \And
   Omar Montasser \\
   Toyota Technological Institute Chicago\\
  Chicago, 60637\\
   \texttt{omar@ttic.edu} \\
   \AND
   Avrim Blum \\
   Toyota Technological Institute Chicago\\
  Chicago, 60637\\
   \texttt{avrim@ttic.edu} \\
}

\DeclareMathOperator*{\argmin}{arg\,min}

\newcommand{\st}{\text{ s.t. }}

\newcommand{\ceil}[1]{\left\lceil#1\right\rceil}

\newtheorem{lemma}{Lemma}
\newtheorem{theorem}{Theorem}
\newtheorem{proposition}{Proposition}
\newtheorem{definition}{Definition}

\newtheorem{example}{Example}
\newcommand{\R}{\mathbb{R}}

\newcommand{\NN}{{\mathbb N}}

\newcommand{\ind}[1]{\mathds{1}[#1]}
\newcommand{\bOne}{{\bm 1}}

\newcommand{\EE}[1]{\mathbb{E}\left[#1\right]}

\newcommand{\EEs}[2]{\mathbb{E}_{#1}\left[#2\right]}
\newcommand{\EEc}[2]{\mathbb{E}\left[#1\left|#2\right.\right]}

\newcommand{\PPs}[2]{\mathbb{P}_{#1}\left(#2\right)}

\newcommand{\norm}[1]{\left\|#1\right\|}

\newcommand{\abs}[1]{\left|#1\right|}

\newcommand{\cA}{\mathcal{A}}
\newcommand{\cB}{\mathcal{B}}

\newcommand{\cD}{\mathcal{D}}

\newcommand{\cF}{\mathcal{F}}
\newcommand{\cG}{\mathcal{G}}
\newcommand{\cH}{\mathcal{H}}
\newcommand{\cI}{\mathcal{I}}

\newcommand{\cK}{\mathcal{K}}
\newcommand{\cL}{\mathcal{L}}

\newcommand{\cM}{\mathcal{M}}

\newcommand{\cP}{\mathcal{P}}

\newcommand{\cS}{\mathcal{S}}
\newcommand{\cT}{\mathcal{T}}

\newcommand{\cU}{\mathcal{U}}
\newcommand{\cV}{\mathcal{V}}

\newcommand{\cX}{\mathcal{X}}

\newcommand{\cY}{\mathcal{Y}}

\newcommand{\bb}{{\bf b}}

\newcommand{\bg}{{\bf g}}

\newcommand{\be}{{\bf e}}
\newcommand{\bff}{{\bf f}}
\newcommand{\bh}{{\bf h}}

\newcommand{\bv}{{\bf v}}

\newcommand{\by}{{\bf y}}

\newcommand{\bx}{{\bf x}}

\renewcommand{\hat}{\widehat}
\renewcommand{\tilde}{\widetilde}

\newcommand{\bphi}{{\boldsymbol \phi}}

\newcommand{\nothere}[1]{}

\newcommand{\err}{\mathrm{err}}
\newcommand{\vcd}{\mathrm{VCdim}}

\newcommand{\trn}{\mathrm{trn}}
\newcommand{\conv}{\mathrm{Conv}}

\newcommand{\A}{\mathcal{A}}

\newcommand{\Majority}{{\mathrm{Majority}}}

\newcommand{\supp}{{\mathrm{supp}}}

\newcommand{\diag}{\mathrm{diag}}

\newcommand{\vcao}{\mathrm{VC_{ao}}}
\newcommand{\vco}{\mathrm{VC_{o}}}

\newcommand{\inv}{\mathrm{INV}}
\newcommand{\re}{\mathrm{RE}}
\newcommand{\ag}{\mathrm{AG}}
\newcommand{\da}{\mathrm{DA}}
\newcommand{\Sym}{\mathrm{Sym}}
\newcommand{\BL}{\mathrm{BL}}

\newcommand\hssout{\bgroup\markoverwith{\textcolor{magenta}{\rule[0.5ex]{2pt}{1pt}}}\ULon}

\renewcommand{\vec}[1]{\boldsymbol{\mathbf{#1}}}

\begin{document}


\maketitle
\begin{abstract}
Transformation invariances are present in many real-world problems.
For example, image classification is usually invariant to rotation and color transformation: a rotated car in a different color is still identified as a car.
Data augmentation, which adds the transformed data into the training set and trains a model on the augmented data, is one commonly used technique to build these invariances into the learning process.
However, it is unclear how data augmentation performs theoretically and what the optimal algorithm is in presence of transformation invariances.
In this paper, we study PAC learnability under transformation invariances in three settings according to different levels of realizability: (i) A hypothesis fits the augmented data; (ii) A hypothesis fits only the original data and the transformed data lying in the support of the data distribution; (iii) Agnostic case.
One interesting observation is that distinguishing between the original data and the transformed data is necessary to achieve optimal accuracy in setting (ii) and (iii), 
which implies that any algorithm not differentiating between the original and transformed data (including data augmentation) is not optimal.
Furthermore, this type of algorithms can even ``harm'' the accuracy.
In setting (i), although it is unnecessary to distinguish between the two data sets, data augmentation still does not perform optimally.
Due to such a difference, we propose two combinatorial measures characterizing the optimal sample complexity in setting (i) and (ii)(iii) and provide the optimal algorithms.
\end{abstract}

\section{Introduction}
Transformation invariances are present in many real-world learning problems.
That is, given a certain set of transformations, the label of an instance is preserved under any transformation from the set.
Image classification is often invariant to rotation/flip/color translation.
Syntax parsing is invariant to exchange of noun phrases in a sentence.
Such invariances are often built into the learning process by two ways.
One is designing new architectures in neural networks to learn a transformation invariant feature, which is usually task-specific and challenging.
A more universally applicable and easier way is data augmentation (DA)\footnote{Throughout the paper, we refer to ERM over the augmented data by DA.}, that is, adding the transformed data into the training set and training a model with the augmented data.
Although DA performs well empirically, it is unclear whether and when DA ``helps''. 
In this paper, we focus on answering two questions:
\begin{center}
    \textit{How does data augmentation perform theoretically? 
    \\What is the optimal algorithm in terms of sample complexity under transformation invariances?}
\end{center}

We formalize the problem of binary classification under transformation invariances in the PAC model.
Given instance space $\cX$, label space $\cY=\{0,1\}$, and hypothesis class $\cH$,
we consider the following three settings according to different levels of realizability.
\begin{itemize}
    \item[(i)] Invariantly realizable setting: There exists a hypothesis $h^*\in \cH$ such that $h^*$ can correctly classify not only the natural data (drawn from the data distribution) but also the transformed data. For example, considering the transformation of rotating images where all natural images are upright, the hypothesis $h^*$ can correctly classify every upright image (natural data) and their rotations (transformed data).
    \item[(ii)] Relaxed realizable setting: There exists a hypothesis $h^*\in \cH$ such $h^*$ has zero error over the support of the data distribution (and therefore will correctly classify the transformed data that lies in the support of the data distribution), but $h^*$ may not correctly classify transformed data that lies outside the support of the natural data distribution. 
    For example, there exists an $h^*$ classifying all small rotations that lie in the support of the distribution correctly, but misclassifying upside-down cars.
    \item[(iii)] Agnostic setting: Every hypothesis in $\cH$ might not fit the natural data.
\end{itemize}
In most of this work, we consider the case where the set of transformations forms a group (e.g., all rotations and all color translations), which is a classic setting studied in literature~\citep[e.g.,][]{cohen2016group, bloem2020probabilistic, Chen2020}.
Some algorithms and analyses in this work also apply to non-group transformations (e.g., croppings).

\textbf{Main contributions} 
First, we show that DA outperforms vanilla ERM but is sub-optimal in setting (i) above.
We then introduce a complexity measure (see Definition~\ref{def:vco}) that characterizes the optimal sample complexity of learning in setting (i), 
and we give an optimal (up to log-factors) algorithm in this setting based on 1-inclusion-graph predictors.
Second, we characterize the complexity of learning in setting (ii) when the learner only receives the augmented data (without specifying which are natural).
Such a characterization provides us with a sufficient condition under which DA "hurts".
Third, we introduce a complexity measure (see Definition~\ref{def:vcao}) that characterizes the optimal sample complexity of learning in settings (ii) and (iii) above, and we give optimal algorithms for these settings. 
Finally, we also provide adaptive learning algorithms that interpolate between settings (i) and (ii), i.e., when $h^*$ is {\em partially} invariant. 
We want to emphasize that our complexity measures take into account the complexity of both the hypothesis class $\cH$ and the set of transformations being considered.
The results are formally summarized in Section~\ref{sec:results}.

\textbf{Related work}
Theoretical guarantees of DA has received a lot of attention recently. 
\cite{Chen2020,Lyle2020} study theoretical guarantees of DA under the assumption of ``equality'' in distribution, i.e., for any transformation in the transformation group,
the data distribution of the transformed data is approximately the same as that of the natural data (e.g., the upside-down variations of images happen at the same probability as the original upright images).
Under this assumption, they show that DA reduces variance and induces better generalization error upper bounds.
Our work does not make such an assumption.
\cite{Dao2019} models augmentation as a Markov process and shows that for kernel linear classifiers, 
DA can be approximated by
first-order feature averaging and second-order variance regularization components.
The concurrent work by~\cite{shen2022data} studies the benefit of DA when training a two layer convolutional neural
network in a specific multi-view model, showing that DA can alter the relative importance of various features. 
There is a line of theoretical study on the invariance gain in different models.
For example, \cite{elesedy2021linear} study the linear model and
\cite{elesedy2021kernel,mei2021learning, bietti2021sample} study the non-parametric regression.
The concurrent work by~\cite{elesedy2022group} also studies PAC learning under transformation invariances but only provides an upper bound on the sample complexity, while our work provides a complete characterization of learning under this model with optimal algorithms.
There is a parallel line of theoretical study on architecture design \citep[e.g.,][]{wood1996representation,ravanbakhsh2017equivariance,kondor2018generalization,bloem2020probabilistic}.

Learning under transformation invariances has also been studied a lot empirically. 
Here we briefly mention a few results.
DA has been applied as standard method in modern deep learning, e.g., in Alexnet~\citep{krizhevsky2012imagenet}.
\cite{gontijo2020affinity} proposes two measures, affinity and diversity, to quantify the performance of the existing DA methods.
\cite{fawzi2016adaptive,cubuk2018autoaugment,chatzipantazis2021learning} study how to automatically search for improved data
augmentation policies.
For architecture design, one celebrated example is convolutions~\citep{fukushima1982neocognitron,lecun1989backpropagation}, which are translation equivariant.
See \cite{cohen2016group, dieleman2016exploiting, worrall2017harmonic} for more different architectures invariant or equivariant to different symmetries.

Another line of related work is adversarial training, which adds the perturbed data into the training set and can be considered as a special type of data augmentation.
\cite{raghunathan2019adversarial,schmidt2018adversarially,nakkiran2019adversarial} study the standard accuracy of adversarial training and provide examples showing that adversarial training can sometimes ``harm'' standard accuracy.

\textbf{Notation}
For any $n\in \NN$, let $\be_1,\be_2,\ldots$ denote the standard basis vectors in $\R^n$.
For any set $\cV$ and any $\bv\in \cV^n$, let $\bv_{-i}=(v_1,\ldots,v_{i-1},v_{i+1},\ldots,v_n)\in \cV^{n-1}$ denote the remaining part of $\bv$ after removing the $i$-th entry and $(v',\bv_{-i})=(v_1,\ldots,v_{i-1},v',v_{i+1},\ldots,v_n)\in \cV^n$ denote the vector after replacing $i$-th entry of $\bv$ with $v'\in \cV$.
Let $\oplus$ denote the bitwise XOR operator. 
For any $h\in \cY^\cX$ and $X=\{x_1,\ldots,x_n\}\subset \cX$, denote $h_{|X}= (h(x_1),\ldots,h(x_n))$ the restriction of $h$ on $X$.
A data set or a sample is a multiset of $\cX\times \cY$.
For any sample $S$, let $S_\cX = \{x|(x, y) \in S\}$ (with multiplicity) and for any distribution $\cD$ over $\cX\times \cY$, for $(x,y)\sim \cD$, let $\cD_\cX$ denote
the marginal distribution of $x$.
For any data distribution $\cD$ and any hypothesis $h$, the expected error $\err_{\cD}(h) := \Pr_{(x,y)\sim \cD}(h(x)\neq y)$.
Denote $\err(h)=\err_{\cD}(h)$ when $\cD$ is clear from the context.
For any sample $S$ of finite size,
$\err_{S}(h) := \frac{1}{\abs{S}}\sum_{(x,y)\in S} \ind{h(x)\neq y}$.
For any sample $S$ of possibly of infinite size, we say $\err_{S}(h) = 0$ if $h(x) = y$ for all $(x,y)\in S$.

\section{Problem setup}\label{sec:setup}
We study binary classification under transformation invariances. We denote by $\cX$ the instance space, $\cY=\{0,1\}$ the label space and $\cH$ the hypothesis class.

\textbf{Group transformations} 
We consider a group $\cG$ of transformations acting on the instance space through a mapping $\alpha: \cG \times \cX \mapsto \cX$, which
is compatible with the group operation.
For convenience, we write $\alpha(g,x) =gx$ for $g\in \cG$ and $x\in \cX$.
For example, consider $\cG = \{e, g_1,g_2, g_3\}$ where $e$ is the identify function and $g_i$ is rotation by $90 i$ degrees. Given an image $x$, $e x = x$ is the original image and $g_1 x$ is the image rotated by $90$ degrees.
The \emph{orbit} of any $x\in \cX$ is the subset of $\cX$ that can be obtained by acting an element in $\cG$ on $x$, $\cG x := \{gx|g\in \cG\}$. 
Note that since $\cG$ is a group, for any $x'\in \cG x$, we have $\cG x' = \cG x$. 
Thus we can divide the instance space $\cX$ into a collection of separated orbits, which does not depend on the data distribution.
Given a (natural) data set $S\subset \cX\times \cY$, we call $\cG S := \{(gx,y)|(x,y)\in S,g\in \cG\}$ the \emph{augmented data set}.

\textbf{Transformation invariant hypotheses and distributions}
To model transformation invariance, we assume that the true labels are invariant over the orbits of natural data.
Formally, for any transformation group $\cG$ and $X\subset\cX$, we say a hypothesis $h$ is \emph{$(\cG,X)$-invariant} if 
\[h(gx) = h(x),\forall g\in \cG, x\in X\,.\]
That is to say, for every $x\in X$, $h$ predicts every instance in the orbit of $x$ the same as $x$.
For any marginal distribution $\cD_\cX$ over $\cX$, we say a hypothesis $h$ is \emph{$(\cG,\cD_\cX)$-invariant} if 
$h(gx) = h(x)$ for all $g\in \cG$, for all $x \in \supp(\cD_\cX)$, i.e., $\Pr_{x\sim \cD_\cX}(\exists x'\in \cG x : h(x')\neq h(x)) = 0$.
We say a distribution $\cD$ over $\cX\times \cY$ is \emph{$\cG$-invariant} if there exists a $(\cG,\cD_\cX)$-invariant hypothesis $f^*$ (possibly not in $\cH$) with $\err_\cD(f^*) = 0$.
We assume that the data distribution is $\cG$-invariant throughout the paper.

\textbf{Realizability of hypothesis class} We consider three settings according to the different levels of realizability of $\cH$: (i) invariantly realizable setting, where there exists a $(\cG, \cD_\cX)$-invariant hypothesis $h^*\in \cH$ with $\err_\cD(h^*) = 0$; (ii) relaxed realizable setting, where there exists a (not necessarily $(\cG, \cD_\cX)$-invariant) hypothesis $h^*\in \cH$ with $\err_\cD(h^*) = 0$; and (iii) agnostic setting, where there might not exist a hypothesis in $\cH$ with zero error.
To understand the difference among the three settings, here is an example.

\begin{example}
    Consider $\cX = \{\pm 1,\pm 2\}$, $\cG=\{e,-e\}$ being the group generated by flipping the sign (i.e., $\cG x = \{x,-x\}$), and the data distribution $\cD$ being the uniform distribution over $\{(1,0),(2,0)\}$.
    If $\cH = \{ h(\cdot) = 0\}$ contains only the all-zero function, then it is in setting (i) as $h(\cdot) = 0$ is $(\cG,\cD_\cX)$-invariant and $\err_\cD(h) = 0$; If $\cH = \{\ind{x<0}\}$ contains only the hypothesis predicting $\{-1, -2\}$ as $1$ and $\{1,2\}$ as $0$, then it is in setting (ii) as $\ind{x<0}$ is not $(\cG,\cD_\cX)$-invariant but $\err_\cD(\ind{x<0}) = 0$; If $\cH = \{\ind{x>0}\}$, it is in setting (iii) as no hypothesis in $\cH$ has zero error.
\end{example}
The following definitions formalize the notion of PAC learning in the three settings.

\begin{definition}[Invariantly realizable PAC learnability]\label{def:inv-pac}
    For any $\epsilon,\delta \in (0,1)$, the sample complexity of invariantly realizable $(\epsilon,\delta)$-PAC learning of $\cH$ with respect to transformation group $\cG$, 
    denoted $\cM_\inv(\epsilon,\delta;\cH,\cG)$, 
    is defined as the smallest $m\in \NN$ for which there exists a learning rule $\cA$ such that, 
    for every $\cG$-invariant data distribution $\cD$ over $\cX\times \cY$ where there exists a \emph{$(\cG, \cD_\cX)$-invariant} predictor $h^*\in \cH$ with zero error, 
    $\err_{\cD}(h^*)=0$, with probability at least $1-\delta$ over $S\sim \cD^m$,
    \[\err_{\cD}(\cA(S))\leq \epsilon\,.\]
    If no such $m$ exists, define $\cM_\inv(\epsilon,\delta;\cH,\cG)=\infty$. 
    We say that $\cH$ is PAC learnable in the invariantly realizable setting with respect to transformation group $\cG$ if $\forall \epsilon,\delta\in (0,1)$, $\cM_\inv(\epsilon,\delta;\cH,\cG)$ is finite.
    For any algorithm $\cA$, denote by $\cM_{\inv}(\epsilon,\delta;\cH,\cG,\cA)$ the sample complexity of $\cA$. 
\end{definition}

\begin{definition}[Relaxed realizable PAC learnability]\label{def:re-pac}
    For any $\epsilon,\delta \in (0,1)$, the sample complexity of relaxed realizable $(\epsilon,\delta)$-PAC learning of $\cH$ with respect to transformation group $\cG$, 
    denoted $\cM_\re(\epsilon,\delta;\cH,\cG)$, 
    is defined as the smallest $m\in \NN$ for which there exists a learning rule $\cA$ such that, 
    for every $\cG$-invariant data distribution $\cD$ over $\cX\times \cY$ where there exists a predictor $h^*\in \cH$ with zero error, 
    $\err_{\cD}(h^*)=0$, with probability at least $1-\delta$ over $S\sim \cD^m$,
    \[\err_{\cD}(\cA(S))\leq \epsilon\,.\]
    If no such $m$ exists, define $\cM_\re(\epsilon,\delta;\cH,\cG)=\infty$. 
    We say that $\cH$ is PAC learnable in the relaxed realizable setting with respect to transformation group $\cG$ if $\forall \epsilon,\delta\in (0,1)$, $\cM_\re(\epsilon,\delta;\cH,\cG)$ is finite.
    For any algorithm $\cA$, denote by $\cM_{\re}(\epsilon,\delta;\cH,\cG,\cA)$ the sample complexity of $\cA$.
\end{definition}

\begin{definition}[Agnostic PAC learnability]\label{def:ag-pac}
    For any $\epsilon,\delta \in (0,1)$, the sample complexity of agnostic $(\epsilon,\delta)$-PAC learning of $\cH$ with respect to transformation group $\cG$, 
    denoted $\cM_\ag(\epsilon,\delta;\cH,\cG)$, 
    is defined as the smallest $m\in \NN$ for which there exists a learning rule $\cA$ such that, 
    for every $\cG$-invariant data distribution $\cD$ over $\cX\times \cY$, with probability at least $1-\delta$ over $S\sim \cD^m$,
    \[\err_{\cD}(\cA(S))\leq \inf_{h\in \cH} \err_\cD(h)+\epsilon\,.\]
    If no such $m$ exists, define $\cM_\ag(\epsilon,\delta;\cH,\cG)=\infty$. 
    We say that $\cH$ is PAC learnable in the agnostic setting with respect to transformation group $\cG$ if $\forall \epsilon,\delta\in (0,1)$, $\cM_\ag(\epsilon,\delta;\cH,\cG)$ is finite.
\end{definition}

\textbf{Data augmentation} One main goal of this work is to analyze the sample complexity of data augmentation.
When we talk of data augmentation (DA) as an algorithm, it actually means \emph{ERM over the augmented data}.
Specifically, given a fixed loss function $\cL$ mapping a data set and a hypothesis to $[0,1]$, and a training set $S_\trn$, DA outputs an $h\in \cH$ such that $h(x)=y$ for all $(x,y)\in \cG S_\trn$ if there exists one; 
outputs a hypothesis $h\in \cH$ with the minimal loss $\cL(\cG S_\trn, h)$ otherwise.
When we say DA without specifying the loss function, it means DA w.r.t. an arbitrary loss function, which can be defined based on any probability measure on the transformation group.

To characterize sample complexities, we define two measures as follows.

\begin{definition}[VC dimension of orbits]
\label{def:vco}
    The VC dimension of orbits, denoted $\vco(\cH, \cG)$, is defined as the largest integer $k$ for which there exists a set 
    $X = \{x_1,\ldots,x_k\}\subset \cX$ such that their orbits are pairwise disjoint, i.e., $\cG x_i \cap \cG x_j =\emptyset, \forall i,j\in [k]$ and every labeling of $X$ is realized by a \emph{$(\cG, X)$-invariant} hypothesis in $\cH$, i.e., $\forall y \in \{0,1\}^k$, there exists a $(\cG, X)$-invariant hypothesis $h\in \cH$ s.t. $h(x_i)=y_i,\forall i\in[k]$.
\end{definition}

\begin{definition}[VC dimension across orbits]
\label{def:vcao}
    The VC dimension across orbits, denoted $\vcao(\cH, \cG)$, is defined as the largest integer $k$ for which there exists a set $X = \{x_1,\ldots,x_k\}\subset \cX$ such that their orbits are pairwise disjoint, i.e., $\cG x_i \cap \cG x_j =\emptyset, \forall i,j\in [k]$ and every labeling of $X$ is realized by a hypothesis in $\cH$, i.e., $\forall y \in \{0,1\}^k$, there exists a hypothesis $h\in \cH$ s.t. $h(x_i)=y_i,\forall i\in[k]$.
\end{definition}
Let $\vcd(\cH)$ denote the VC dimension of $\cH$. 
By definition, it is direct to check that $\vco(\cH,\cG)\leq \vcao(\cH,\cG)\leq \vcd(\cH)$. 
For any $\cH,\cG$ with $\vco(\cH,\cG)=d$,
we can supplement $\cH$ to a new hypothesis class $\cH'$ such that $\vco(\cH',\cG)$ is still $d$ while $\vcao(\cH',\cG)$ is as large as the total number of orbits with at least two instances, i.e., $\vcao(\cH',\cG)=\abs{\{\cG x|\abs{\cG x}\geq 2, x\in \cX\}}$.
This can be done by supplementing $\cH$ with all hypotheses predicting $\cG x$  with two different labels for all $x$ with $\abs{\cG x}\geq 2$.
Besides, for any $\cH$ with $\vcd(\cH)=d$, we can construct a transformation group $\cG$ to make all instances lie in one single orbit, which makes $\vcao(\cH,\cG)\leq 1$.
Hence the gap among the three measures can be arbitrarily large.
Here are a few examples for better understanding of the gaps.

\begin{example}\label{eg:vco0}
    Consider $\cX = \{\pm 1, \pm 2,\ldots, \pm 2d\}$ for some $d>0$ , $\cH= \{\ind{x\in A}|A\subset [2d] \text{ and } \abs{A} = d\}$ being the set of all hypotheses labeling exact $d$ elements from $[2d]$ as $1$ and $\cG = \{e,-e\}$ being the group generated by flipping the sign.
    Then we have $\vco(\cH,\cG) = 0$ since for any $i\in [2d]$, there is no $(\cG, \{i\})$-invariant hypothesis that can label $i$ as $1$ (which is due to the fact that $-i$ is labeled as $0$ by any hypothesis in $\cH$).
    It is direct to check that $\vcao(\cH,\cG) = \vcd(\cH) = d$.
\end{example}

\begin{example}
    Consider $\cX = \{x\in \R^2|\norm{x}_2 = 1\}$ being a circle, the hypothesis class $\cH = \{0,1\}^{\cX}$ being all labeling functions and $\cG$ being all rotations (thus $\forall x, \cG x = \cX$).
    Then we have $\vco(\cH,\cG)=\vcao(\cH,\cG)= 1$ as there is only one orbit and $\vcd(\cH) = \infty$. 
\end{example}
\begin{example}
     Consider the natural data being $k$ upright images and the transformation set $\cG$ is rotation by $0, 360/n, 2\cdot 360/n, \ldots, (n-1)\cdot 360/n$ degrees for some integer n. For an expressive hypothesis class $\cH$ (e.g., neural networks) that can shatter all rotated versions of these images, we have $\vcd(\cH)=nk$ and $\vco(\cH,\cG) = \vcao(\cH,\cG) = k$.
     For a hypothesis class $\cH'$ composed of all hypotheses labeling all upright images and their upside-down variations differently, we have $\vcd(\cH') = (n-1)k$, $\vcao(\cH',\cG) = k$ and $\vco(\cH',\cG) = 0$.
\end{example}


\section{Main results}\label{sec:results}
We next present and discuss our main results.
\begin{itemize}[leftmargin=*]
    \item Invariantly realizable setting (Definition~\ref{def:inv-pac})
    \begin{itemize}[leftmargin=*]
        \item \textbf{DA ``helps'' but is not optimal. The sample complexity of DA is characterized by $\vcao(\cH,\cG)$.} For any $\cH,\cG$, DA can learn $\cH$ with sample complexity $\tilde O(\frac{\vcao(\cH,\cG)}{\epsilon}+\frac{1}{\epsilon}\log \frac 1 \delta)$, where $\tilde O$ ignores log-factors of $\frac{\vcao(\cH,\cG)}{\epsilon}$ (Theorem~\ref{thm:inv-da-ub}). 
        For all $d> 0$, there exists $\cH, \cG$ with $\vcao(\cH, \cG)=d$ and $\vco(\cH,\cG) = 0$ such that DA needs $\Omega(\frac{d}{\epsilon})$ samples (Theorem~\ref{thm:inv-da-lb}).
        \item
        \textbf{The optimal sample complexity is characterized by $\vco(\cH,\cG)$.} 
        For any $\cH,\cG$, we have $\Omega(\frac{\vco(\cH,\cG)}{\epsilon} +\frac{1}{\epsilon}\log\frac{1}{\delta})\leq \cM_\inv(\epsilon,\delta;\cH,\cG) \leq O(\frac{\vco(\cH,\cG)}{\epsilon}\log \frac{1}{\delta})$ (Theorem~\ref{thm:inv-opt}).
        We propose an algorithm achieving this upper bound based on 1-inclusion graphs, which does not distinguish between the original and transformed data.
        It is worth noting that the algorithm takes the invariance over the test point into account, which provides some theoretical justification for test-time adaptation such as~\cite{wang2020tent}.
    \end{itemize}

    \item Relaxed realizable setting (Definition~\ref{def:re-pac})
    \begin{itemize}[leftmargin=*]
        \item \textbf{DA can ``hurt''.} DA belongs to the family of algorithms not distinguishing the original data from the transformed data.
        We show that the optimal sample complexity of this family is characterized by $\mu(\cH,\cG)$ (see Definition~\ref{def:mu}) (Theorem~\ref{thm:re-da}), which can be arbitrarily larger than $\vcd(\cH)$.
        This implies that for any $\cH,\cG$ with $\mu(\cH,\cG)>\vcd(\cH)$, 
        the sample complexity of DA is higher than that of ERM.

        \item \textbf{The optimal sample complexity is characterized by $\vcao(\cH,\cG)$.}
        For any $\cH,\cG$, we have $\Omega(\frac{\vcao(\cH,\cG)}{\epsilon} + \frac{\log(1/\delta)}{\epsilon})\leq \cM_\re(\epsilon,\delta;\cH,\cG) \leq \tilde O(\frac{\vcao(\cH,\cG)}{\epsilon}+\frac{1}{\epsilon}\log \frac 1 \delta)$ (Theorem~\ref{thm:re-opt}).
        We propose two algorithms achieving similar upper bounds, with one based on ERM and one based on 1-inclusion graphs.
        Both algorithms have to distinguish between the original and the transformed data.
        \item \textbf{An adaptive algorithm interpolates between two settings.} We present an algorithm that adapts to different levels of invariance of the target function $h^*$, which achieves $\tilde O(\frac{\vcao(\cH,\cG)}{\epsilon})$ sample complexity in the relaxed realizable setting and $\tilde O(\frac{\vco(\cH,\cG)}{\epsilon})$ sample complexity in the invariantly realizable setting without knowing it (Theorem~\ref{thm:re-eta-unknown-ub} in Appendix). 
    \end{itemize}
    \item Agnostic setting (Definition~\ref{def:ag-pac})
    \begin{itemize}[leftmargin=*]
    \item \textbf{The optimal sample complexity is characterized by $\vcao(\cH,\cG)$.}
    For any $\cH,\cG$, $\cM_\ag(\epsilon,\delta;\cH,\cG) = O\left(\frac{\vcao(\cH,\cG)}{\epsilon^2}\log^2\left(\frac{\vcao(\cH,\cG)}{\epsilon}\right)+\frac{1}{\epsilon^2}\log(\frac 1 \delta)\right)$ (Theorem~\ref{thm:ag-opt}).
    Since $\cM_\ag(\epsilon,\delta;\cH,\cG)  \geq \cM_\re(\epsilon,\delta;\cH,\cG)$, $\vcao(\cH,\cG)$ characterizes the optimal sample complexity.
\end{itemize}
\end{itemize}
\section{Invariantly Realizable setting}\label{sec:inv}
In this section, we discuss the results in the invariantly realizable setting (see Definition~\ref{def:inv-pac}).

\subsection{DA ``helps'' but is not optimal}
We show that in the invariantly realizable setting, DA indeed ``helps'' to improve the sample complexity from $\tilde O(\frac{\vcd(\cH)}{\epsilon})$ (the sample complexity of ERM in standard PAC learning) to $\tilde O(\frac{\vcao(\cH,\cG)}{\epsilon})$.
First, we have the following upper bound on the sample complexity of DA.
\begin{theorem}\label{thm:inv-da-ub}
    For any $\cH,\cG$ with $\vcao(\cH,\cG)<\infty$, DA satisfies that $\cM_\inv(\epsilon,\delta;\cH,\cG,\da) = O( \frac{\vcao(\cH,\cG)}{\epsilon}\log^3\frac{\vcao(\cH,\cG)}{\epsilon}+\frac{1}{\epsilon}\log \frac 1 \delta)$.
\end{theorem}
Intuitively, for a set of instances in one orbit that can be labeled by $\cH$ in multiple ways, we only need to observe one instance from this orbit to learn the labels of all the instances by applying DA. Thus, DA helps to improve the accuracy.
The detailed proof is deferred to Appendix~\ref{app:inv-da-ub}.
However, DA does not fully exploit the transformation invariances as it only utilizes the invariances of the training set.
Hence, DA does not perform optimally in presence of the transformation invariances.
In fact, besides DA, all proper learners (i.e., learners outputting a hypothesis in $\cH$) have the same problem.

\begin{theorem}\label{thm:inv-da-lb}
    For any $d>0$, there exists a hypothesis class $\cH_d$ and a group $\cG_d$ with $\vcao(\cH_d,\cG_d)=d$ and $\vco(\cH_d,\cG_d) = 0$ such that $\cM_\inv(\epsilon,\frac{1}{9};\cH_d,\cG_d,\cA)=\Omega( \frac{d}{\epsilon})$ for any proper learner $\cA$, including DA and standard ERM.
\end{theorem}
The theorem shows that DA is sub-optimal as we will show that the optimal sample complexity is characterized by $\vco(\cH,\cG)$ in Theorem~\ref{thm:inv-opt}. 
We provide an idea of the construction here and defer the detailed proof to Appendix~\ref{app:inv-da-lb}.
Consider the $\cX$, $\cH$ and $\cG$ in Example~\ref{eg:vco0}.
Pick the target function $\ind{x\in A}$ uniformly at random from $\cH$ and let the data distribution only put probability mass on points in $[2d]\setminus A$, the orbits of which are labeled as $0$ by the target function.
Then any proper learner must predict $d$ unobserved examples of $[2d]$ as $1$, which leads to high error if the learner observes fewer than $d/2$ examples.
Theorem~\ref{thm:inv-da-lb} also implies that for any hypothesis class including $\cH$ as a subset,
there exists a DA learner (i.e., a proper learner fitting the augmented data) whose sample complexity is $\Omega(\frac{d}{\epsilon})$.

Theorem~\ref{thm:inv-da-ub} shows that the sample complexity of DA is $\tilde O(\frac{\vcao(\cH,\cG)}{\epsilon})$, better than that of ERM in standard PAC learning, $\tilde O(\frac{\vcd(\cH)}{\epsilon})$.
This is insufficient to show that DA outperforms ERM as it might be possible that ERM can also achieve better sample complexity in presence of transformation invariances.
To illustrate that DA indeed "helps", we show that any algorithm without exploiting the transformation invariances still requires sample complexity of $\Omega(\frac{\vcd(\cH)}{\epsilon})$.
\begin{theorem}\label{thm:inv-da-help}
    For any $\cH$, there exists a group $\cG$ with $\vcao(\cH,\cG) \leq 5$ s.t. $\cM_\inv(\epsilon,\delta;\cH,\cG,\cA)=\Omega(\frac{\vcd(\cH)}{\epsilon} +\frac{1}{\epsilon}\log \frac{1}{\delta})$ for any algorithm $\cA$ not given any information about $\cG$ (e.g., ERM).
\end{theorem}
The basic idea is that, given a set of $k$ instances that can be shattered by $\cH$ for some $k>0$,
$\cG$ is uniformly at random picked from a set of $2^k$ groups, each of which partitions the set into two orbits in a different way.
If given $\cG$, the algorithm only need to observe one instance in each orbit to learn the labels of all $k$ instances.
If not, the algorithm can only randomly guess the label of every unobserved instance.
The detailed construction is included in Appendix~\ref{app:inv-da-help}.

\subsection{The optimal algorithm}\label{subsec:inv-opt}

We show that the optimal sample complexity is characterized by $\vco(\cH,\cG)$.
\begin{theorem}\label{thm:inv-opt}
    For any $\cH, \cG$ with $\vco(\cH,\cG)<\infty$, we have $\Omega(\frac{\vco(\cH,\cG)}{\epsilon} +\frac{1}{\epsilon}\log\frac{1}{\delta})\leq \cM_{\inv}(\epsilon, \delta;\cH,\cG)\leq O(\frac{\vco(\cH,\cG)}{\epsilon}\log \frac{1}{\delta})$.
\end{theorem}
Our algorithm is based on the 1-inclusion-graph predictor by~\cite{haussler1994predicting}.
Given hypothesis class $H$ and instance space $X=\{x_1,\ldots,x_t\}$, the classical 1-inclusion-graph consists of vertices $\{h_{|X}|h\in H\}$, which are labelings of $X$ realized by $H$, and two vertices are connected by an edge if and only if they differ at the labeling of exactly a single $x_i\in X$.
\cite{haussler1994predicting} shows that the edges can be oriented such that each vertex has in-degree at most $\vcd(H)$.
This orientation can be translated to a prediction rule.
Specifically, for any $i\in [t]$, given the labels of all instances in $X$ except $x_i$, if there are two hypotheses $h,h'\in H$ such that their labelings are consistent with the labels of $X\setminus\{x_i\}$ and different at $x_i$, then $h_{|X},h'_{|X}$ are two vertices in the graph and we predict the label of $x_i$ as the edge between $h_{|X},h'_{|X}$ is oriented against.
The average leave-one-out-error is upper bounded by $\frac{\vcd(H)}{t}$.
\begin{lemma}[Theorem~2.3 of \cite{haussler1994predicting}]\label{lmm:1-inclusion}
For any hypothesis class $H$ and instance space $X$ with $\vcd(H)<\infty$, there is a function $Q: (X\times \cY)^*\times X\mapsto \cY$ such that, for any $t\in \NN$ and sample $\{(x_1,y_1),\ldots,(x_t,y_t)\}$ that is realizable w.r.t. $H$,
\begin{align}
    \frac{1}{t!}\sum_{\sigma\in \Sym(t) } \ind{Q(\{x_{\sigma(i)},y_{\sigma(i)}\}_{i\in [t-1]},x_{\sigma(t)})\neq y_{\sigma(t)}}\leq \frac{\vcd(H)}{t}\,,\label{eq:1-inclusion}
\end{align}
where $\Sym(t)$ denotes the symmetric group on $[t]$. The function $Q$ can be constructed by a 1-inclusion-graph predictor.
\end{lemma}
Denote by $Q_{H,X}$ the function guaranteed by Eq~\eqref{eq:1-inclusion} for hypothesis class $H$ and instance space $X$.
For any $t\in \NN$ and $S = \{(x_1,y_1),\ldots,(x_t,y_t)\}$, 
let $X_S$ denote the set of different elements in $S_\cX$.
Define $\cH(X_S):=\{h_{|X_S}|h\in \cH \text{ is } (\cG,X_S) \text{-invariant}\}$ being the set of all possible $(\cG,X_S)$-invariant labelings of $X_S$.
We then define our algorithm $\cA(S)$ by letting $\cA(S)(x) = Q_{\cH(X_S\cup \{x\}), X_{S}\cup \{x\}}(S, x)$ if $\cH(X_S\cup \{x\})\neq \emptyset$ and predicting arbitrarily if $\cH(X_S\cup \{x\})= \emptyset$.
That is to say, $\cA(S)$ needs to construct a function $Q$ for every test example. 
Given any test example, this 1-inclusion-graph-based algorithm takes into account whether the prediction can be invariant over the whole orbit of the test example and thus benefits from the invariance of test examples.
This can provide some theoretical justification for test-time adaptation such as~\cite{wang2020tent}.
By definition, we have $\vcd(\cH(X_S\cup \{x\}))\leq \vco(\cH,\cG)$ for all $S$ and $x$.
Then the expected error of $\cA$ can be bounded by $\vco(\cH,\cG)$ through Lemma~\ref{lmm:1-inclusion}.
We defer the details and the proof of Theorem~\ref{thm:inv-opt} to Appendix~\ref{app:inv-opt}.
Note that the results of Theorem~\ref{thm:inv-opt} also apply to non-group transformations\footnote{In this case, we only assume that $\cG$ contains the identity element.}.

\section{Relaxed realizable setting}\label{sec:re}
In this section, we discuss the results in the relaxed realizable setting (see Definition~\ref{def:re-pac}).
As we can see, DA belongs to the family of algorithms not distinguishing between the original and transformed data. 
In Section~\ref{subsec:dahurts}, we provide a tight characterization $\mu(\cH,\cG)$ (Definition~\ref{def:mu}) on the sample complexity of this family algorithms.
This implies that when $\mu(\cH,\cG)>\vcd(\cH)$, there exists a distribution s.t. DA performs worse than ERM.
We then show that there exists $\cH,\cG$ such that $\mu(\cH,\cG)>\vcd(\cH)$ and the gap can be arbitrarily large.
In Section~\ref{subsec:re-opt}, we provide two optimal algorithms, both of which have to distinguish between the original and transformed data.

\subsection{DA can even ``hurt''}\label{subsec:dahurts}
In the invariantly realizable setting, the optimal algorithm based on 1-inclusion graphs does not need to distinguish between the original and transformed data since $\cH(X_S\cup \{x\})$ in the algorithm is fully determined by the augmented data.
However, in the relaxed realizable setting, distinguishing between the original and transformed data is crucial.
In the following, we will provide a characterization of the sample complexity of algorithms not distinguishing between the original and transformed data, including DA.
Such a characterization induces a sufficient condition when DA ``hurts''.

Let $\cM_\da (\epsilon,\delta;\cH,\cG)$ be the smallest integer $m$ for which there exists a learning rule $\cA$ such that for every $\cG$-invariant data distribution $\cD$, with probability at least $1-\delta$ over $S_\trn\sim \cD^m$, $\err_\cD(\cA(\cG S_\trn))\leq \epsilon$.
The quantity $\cM_\da (\epsilon,\delta;\cH,\cG)$ is the optimal sample complexity achievable if algorithms can only access the augmented data without knowing the original training set.
In standard PAC learning, the optimal sample complexity can be characterized by the maximum density of any subgraph of the 1-inclusion graphs, which is actually equal to the VC dimension~\citep{haussler1994predicting, daniely2014optimal}.
Analogously, we characterize $\cM_\da (\epsilon,\delta;\cH,\cG)$ based on a variant of 1-inclusion graphs, which is constructed as follows.

In the 1-inclusion graph for standard PAC learning~\citep{haussler1994predicting}, given any sequence of instances $\bx = (x_1,\ldots,x_t)$, the vertices are labelings of $\vec x$ and two vertices are connected iff. they are different at only one instance in $\vec x$ and this instance appears once.
In our setting, the input is a multiset of labeled orbits and an unlabeled test instance, hence the vertices are pairs of labelings of orbits and unlabeled instances.
Specifically, for any $t\in \NN$,
given a multi-set of orbits $\bphi = \{\phi_1,\ldots,\phi_t\}$ of some unknown original data,
a labeling $\bff\in \cY^t$ is possible iff. there exists a sequence of instances $\bx = (x_1,\ldots,x_t) \in \prod_{i=1}^t \phi_i$ and a hypothesis $h\in \cH$ such that $h_{|\bx} = \bff$ and that instances in the same orbit are labeled the same,
i.e., $(\cG x_i\times \{1-f_i\})\cap (\{(x_j,f_j)\}_{j\in [t]})=\emptyset$ for all $i\in [t]$.
We denote the set of all possible labelings of $\bphi$ by
\begin{small}
\begin{equation}\label{eq:def-Pi}
    \Pi_\cH(\bphi) := \{\bff\in \cY^t| \exists h\in \cH, \exists \bx\in \prod_{i=1}^t \phi_i, h_{|\bx} = \bff \text{ and }\cup_{i\in [t]}(\cG x_i\times \{1-f_i\})\cap \{(x_j,f_j)|j\in [t]\}=\emptyset\}\,.
\end{equation}
\end{small}
Denote the set of all such sequences $\bx \in \prod_{i=1}^t \phi_i$ of instances that can be labeled as $\bff$ by
\[\cU_{\bff}(\bphi) := \{\bx\in  \prod_{i=1}^t \phi_i|\exists h\in \cH, h_{|\bx} = \bff \text{ and }\cup_{i\in [t]}(\cG x_i\times \{1-f_i\})\cap \{(x_j,f_j)|j\in [t]\}=\emptyset\}\,.\]
Denote the set of all pairs of labeling and its corresponding instance sequence by
\begin{equation}\label{eq:def-B}
    B(\cH,\cG,\bphi):=\cup_{\bff\in \Pi_\cH(\bphi)}\{\bff\}\times \cU_{\bff}(\bphi)\,.
\end{equation}
For any $(\bff,\vec x)\in B(\cH,\cG,\bphi)$, $\vec x$ is a candidate of original data and $\bff$ is a candidate of labeling of $\vec x$.
Now we define a graph $G_{\cH,\cG}(\bphi)$, 
where the vertices are all pairs of labeling $\bff \in \Pi_\cH(\bphi)$ and an element in a instance sequence corresponding to $\bff$. Formally, the vertex set is
\[V = \{(\bff,x_i)|(\bff, \bx)\in B(\cH,\cG,\bphi), i\in [t]\}\,.\]
For every two vertices $(\bff,x)$ and $(\bg,z)$, they are connected if and only if (i) $x=z$; (ii) there exists $j\in[t]$ such that $x\in \phi_j$, $f_i =g_i, \forall i\neq j$ and $f_j \neq g_j$; 
and (iii) $\phi_j$ only appear once in $\bphi$. 
Each edge can be represented by $e=\{\bff,\bg,x\}$ and we denote $E$ the edge set. 
If an edge $e=\{\bff,\bg,x\}$ exists, the edge could be recovered given only $\bff, x$ or $\bg,x$, and thus, we also denote by $e(\bff,x) =e(\bg,x) =\{\bff,\bg,x\}$.
Any algorithm accessing only the augmented data corresponds to an orientation of edges in the graph we constructed, which leads to the following definition.
    
\begin{definition}\label{def:mu}
    Let $w: E\times \Pi_\cH(\bphi)\mapsto [0,1]$ be a mapping such that for every $e=\{\bff,\bg,x\}, w(e,\bff)+w(e,\bg)=1$ and $w(e,\bh)=0$ if $\bh\notin e$ and let $W$ be the set of all such mappings.
    Note that $w$ actually defines a randomized orientation of each edge in graph $G_{\cH,\cG}(\bphi)$: the edge $e=\{\bff,\bg,x\}$ is oriented towards vertex $(\bff,x)$ with probability $w(e,\bff)$.
    For any $(\bff,\bx)\in B(\cH,\cG,\bphi)$, it corresponds to a cluster of vertices $\{(\bff, x_i)|i\in [t]\}$ in $G_{\cH,\cG}(\bphi)$ and $\sum_{i\in [t]: \exists e\in E, \{\bff,x_i\}\subset e} w(e(\bff,x_i),\bff)$ is the expected in-degree of the cluster. 
    Let $\Delta(B(\cH,\cG,\bphi))$ denote the set of all distributions over $B(\cH,\cG,\bphi)$. 
    For any $P\in \Delta(B(\cH,\cG,\bphi))$, we define 
    \begin{equation}\label{eq:mu-P}
        \mu(\cH,\cG,\bphi,P) :=  \min_{w\in W}\EEs{(\bff,\bx)\sim P}{ \sum_{i\in [t]: \exists e\in E, \{\bff,x_i\}\subset e} w(e(\bff,x_i),\bff)}\,.
    \end{equation}
    By taking the supremum over $P$, we define
    $\mu(\cH,\cG,\bphi) := \sup_{P\in \Delta(B(\cH,\cG,\bphi))} \mu(\cH,\cG,\bphi,P)$. By taking supremum over $\bphi$, we define $\mu(\cH,\cG,t) := \sup_{\bphi:\abs{\bphi}=t} \mu(\cH,\cG,\bphi)$
    and
    \begin{equation}
        \mu(\cH,\cG) := \sup_{t\in \NN} \mu(\cH,\cG,t)\,.
    \end{equation}
\end{definition}
\begin{theorem}\label{thm:re-da}
    For any $\cH,\cG$, $\cM_{\da}(\epsilon,\delta;\cH,\cG)$ satisfies the following bounds:
    \begin{itemize}[leftmargin = *]
        \item For all $t\geq 2$ with $\mu(\cH,\cG,t)<\infty$, $\cM_{\da}(\epsilon,\frac{\mu(\cH,\cG,t)}{16(t-1)};\cH,\cG)=\Omega( \frac{\mu(\cH,\cG,t)}{\epsilon})$.
        This implies that if $\mu(\cH,\cG)<\infty$, there exists a constant $c$ dependent on $\cH,\cG$ s.t. $\cM_{\da}(\epsilon,c;\cH,\cG)=\Omega( \frac{\mu(\cH,\cG)}{\epsilon})$.
        \item For all $t$ with $\frac{1}{6}\leq \frac{\mu(\cH,\cG,t)}{t}\leq \frac{1}{3}$, $\cM_{\da}(\epsilon,\delta;\cH,\cG)=O( \frac{\mu(\cH,\cG,t)}{\epsilon} \log^2\frac{\mu(\cH,\cG,t)}{\epsilon} + \frac{1}{\epsilon}\log \frac{1}{\delta})$ .
        \item If $\mu(\cH,\cG)<\infty$, $\cM_{\da}(\epsilon,\delta;\cH,\cG)=O( \frac{\mu(\cH,\cG)\log(1/\delta)}{\epsilon})$.
    \end{itemize}
\end{theorem}
Theorem~\ref{thm:re-da} implies that when $\mu( \cH,\cG)>\vcd(\cH)$, there exists a distribution such that 
any algorithm not differentiating between the original and transformed data performs worse than simply applying ERM over the original data.
We defer the proof of Theorem~\ref{thm:re-da} to Appendix~\ref{app:re-da}.
As we can see, the definition of $\mu(\cH,\cG)$ is not intuitive and it might be difficult to calculate $\mu(\cH,\cG)$ as well as to further determine when $\mu(\cH,\cG)>\vcd(\cH)$. 
We introduce a new dimension as follows, which lower bounds $\mu(\cH,\cG)$ and is easier to calculate.
\begin{definition}[VC Dimension of orbits generated by $\cH$]\label{def:dim}
    The VC dimension of orbits generated by $\cH$, denoted $\dim(\cH, \cG)$, 
    is defined as the largest integer $k$ for which there exists a set $X=\{x_1,\ldots,x_k\}\subset \cX$ such that (i) their orbits $\bphi = \{\phi_1,\ldots,\phi_k\}$ are pairwise disjoint, (ii) $\Pi_\cH(\bphi) = \cY^k$ (defined in Eq~\eqref{eq:def-Pi}) and (iii) there exists a set $B=\{(\bff, \bx_\bff)\}_{\bff\in \cY^k}\subset B(\cH,\cG,\bphi)$ (defined in Eq~\eqref{eq:def-B}) such that $\bff \oplus \bg = \be_i$ implies $x_{\bff,i} = x_{\bg,i}$ for all $i\in [k], \bff,\bg\in \cY^k$.
\end{definition}
\begin{theorem}\label{thm:re-mu-dim}
    For any $\cH,\cG$, $\mu(\cH,\cG)\geq {\dim(\cH,\cG)}/{2}$.
\end{theorem}
The proof is included in Appendix~\ref{app:re-mu-dim}.
Through this dimension, 
we claim that the gap between $\mu(\cH,\cG)$ and $\vcd(\cH)$ can be arbitrarily large.
In the following, we give an example of $\cH,\cG$ with $\dim(\cH,\cG) \gg \vcd(\cH)$, which cannot be learned by DA but can be easily learned by ERM.
\begin{example}
    For any $d>0$, let $\cX = \{\pm 1\}\times \{\be_i|i\in [2d]\}\subset \R^{2d+1}$, $\cH = \{x_1> 0, x_1\leq 0\}$ and $\cG = \{I_{2d+1}, I_{2d+1}-2\diag(\be_1)\}$ (i.e., the cyclic group generated by flipping the sign of $x_1$).
    It is easy to check that $\vcd(\cH) = 1$.
    Let $X = \{(1,\be_1), (1,\be_2),\ldots,(1,\be_{2d})\}$ and then the orbits generated from $X$ are $\bphi = \{\{(-1,\be_i),(1,\be_i)\}|i\in [2d]\}\}$.
    For every labeling $\bff\in \cY^{2d}$, if $\sum_{i\in [2d]}f_i$ is odd, let $\bx_\bff = ((2f_i-1,\be_i))_{i=1}^{2d}$; if $\sum_{i\in [2d]}f_i$ is even, let $\bx_f = ((1-2f_i,\be_i))_{i=1}^{2d}$.
    It is direct to check that $(\bff,\bx_\bff)\in B(\cH,\cG,\bphi)$ for all $\bff\in \cY^{2d}$.
    Then for all $i\in [2d]$, $\bff\oplus \bg=\be_i$ implies $x_{\bff,i}=x_{\bg,i}$.
    Hence, $\dim(\cH,\cG) = 2d$, where $d$ can be an arbitrary positive integer.
    According to Theorem~\ref{thm:re-mu-dim}, we have $\mu(\cH,\cG)\geq d$.
\end{example}
The above example can be interpreted in a vision scenario. Let's consider an example of classifying land birds versus water birds.
The natural data is $2d$ images of land birds with land background and water birds with water background. The transformation set is composed of keeping the current background and changing the background from land (water) to water (land). 
Consider simple hypotheses depending on backgrounds only.
Specifically, $\cH = \{h_1,h_2\}$ with $h_1$ predicting all images with water background as water birds and $h_2$ predicting all images with water background as land birds.
Let the data distribution be the uniform distribution over all the original images.
Then given any training data, $h_1$ and $h_2$ have the same empirical loss on the augmented training data.
Thus, for any unobserved image, DA will make a mistake with constant probability. Hence DA requires at least $\Omega(d)$ sample complexity.
It is direct to check that standard ERM only needs one labeled instance to achieve zero error.

\textbf{Open question: }
It is unclear whether $\mu(\cH,\cG)$ is upper bounded by $\dim(\cH,\cG)$.
If true, then we can tightly characterize $\cM_\da (\epsilon,\delta;\cH,\cG)$ by $\dim(\cH,\cG)$.

\subsection{The optimal algorithms}\label{subsec:re-opt}
Different from the invariantly realizable setting, the optimal sample complexity in the relaxed realizable setting is characterized by $\vcao(\cH,\cG)$.
The optimal (up to log-factors) sample complexity can be achieved by another variant of 1-inclusion-graph predictor.
Besides, we propose an ERM-based algorithm, called ERM-INV (see Appendix~\ref{app:re-opt} for details), achieving the similar guarantee.

\begin{theorem}\label{thm:re-opt}
    For any $\cH,\cG$ with $\vcao(\cH,\cG)<\infty$, we have $\Omega(\frac{\vcao(\cH,\cG)}{\epsilon} + \frac{\log(1/\delta)}{\epsilon})\leq \cM_{\re}(\epsilon,\delta;\cH,\cG) \leq O\left(\min\left( \frac{\vcao(\cH,\cG)}{\epsilon}\log^3\frac{\vcao(\cH,\cG)}{\epsilon}+\frac{1}{\epsilon}\log \frac 1 \delta, \frac{\vcao(\cH,\cG)}{\epsilon}\log \frac{1}{\delta}\right)\right)$.
\end{theorem}
We defer the details of algorithms and the proof of Theorem~\ref{thm:re-opt} to Appendix~\ref{app:re-opt}.
Usually, ERM-INV is more efficient than the 1-inclusion-graph predictor.
But the 1-inclusion-graph predictor as well as the lower bound can apply to non-group transformations.
Another advantage of 1-inclusion-graph predictor is allowing us to design an adaptive framework which automatically adjusts to different levels of invariance of $h^*$.
Specifically, for any hypothesis $h\in \cH$, we say $h$ is $(1-\eta)$-invariant over the distribution $\cD_{\cX}$ for some $\eta\in [0,1]$ if $\PPs{x\sim \cD_{\cX}}{\exists x'\in \cG x, h(x')\neq h(x)}= \eta$.
When $\eta(h^*)=0$, it degenerates into the invariantly realizable setting, which implies that we can achieve better bounds when $\eta(h^*)$ is smaller.
We propose an adaptive algorithm with sample complexity dependent on $\eta(h^*)$ and the details are included in Appendix~\ref{app:unified-re}.

\section{Agnostic setting}
In the agnostic setting (Definition~\ref{def:ag-pac}), $\inf_{h\in \cH} \err(h)$ is possibly non-zero.
Different from the agnostic setting in the standard PAC learning allowing probabilistic labels, our problem is limited to deterministic labels 
because we assume that the data distribution is $\cG$-invariant, i.e., there exists a $(\cG,\cD_\cX)$-invariant hypothesis $f^*$ (possibly not in $\cH$) with $\err_\cD(f^*) = 0$.

\begin{theorem}\label{thm:ag-opt}
    The sample complexity in the agnostic setting satisfies:
    \begin{itemize}[leftmargin = *]
        \item For all $d>0$, there exists $\cH,\cG$ with $\vcao(\cH,\cG)=d$, $\cM_\ag(\epsilon,1/64;\cH,\cG)= \Omega(\frac{d}{\epsilon^2})$.
        \item For any $\cH,\cG$ with $\vcao(\cH,\cG)<\infty$, $\cM_\ag(\epsilon,\delta;\cH,\cG)= \tilde O\left(\frac{\vcao(\cH,\cG)}{\epsilon^2}+\frac{1}{\epsilon^2}\log(\frac 1 \delta)\right)$.
    \end{itemize}
\end{theorem}
For upper bound, we show that ERM-INV achieves sample complexity $\tilde O(\frac{\vcao(\cH,\cG)}{\epsilon^2})$.
There is another way of achieving similar upper bound based on applying the reduction-to-realizable technique of \cite{david2016supervised}.
Note that a direct combination of any reduction-to-realizable technique and any optimal algorithm in relaxed realizable setting does not work in our agnostic setting.
This is because the relaxed realizable setting requires not only realizability, but also invariance in the support of the data distribution.
For example, the reduction method of \cite{hopkins2021realizable} needs to run a realizable algorithm over a set labeled by each $h\in \cH$, which might label two instances in the same orbit differently and make the realizable algorithm not well-defined.
When combining the reduction method of \cite{david2016supervised} and the 1-inclusion-graph-type algorithm, the similar problem also exists but can be fixed by predicting arbitrarily when the invariance property is not satisfied.
For lower bound, According to \cite{ben2014sample}, the sample complexity of agnostic PAC learning under deterministic labels is not fully determined by the VC dimension.
Following the construction by \cite{ben2014sample}, we provide an analogous lower bound in our setting.
The algorithm details and the proofs are deferred to Appendix~\ref{app:ag-opt}.
Analogous to the realizable setting, we provide one algorithm adapting to different levels of invariance of the optimal hypothesis in $\cH$ in Appendix~\ref{app:unified-ag}.
Similar to the results in the realizable settings, the lower bound and the 1-inclusion-graph predictor in the agnostic setting also apply to non-group transformations.

\section{Discussion}
\textbf{Definition of invariance under probabilistic labels} 
In this work, we model invariance by assuming that the data distribution is $\cG$-invariant, which restricts the labels to be deterministic.
It is unclear what ``invariance under probabilistic labels'' means.
One option is assuming that the distribution of the labels is invariant over the orbits, $\Pr(y|g x)=\Pr(y|x)$ for all $g\in \cG, x\in \cX$.
However, such a condition may not characterize invariance in real-world scenarios due to classes having different underlying distributions.
For example, given a fuzzy image with probability $0.5$ being a car and $0.5$ being a tree, it is uncertain if the chance of this image being a car is still $0.5$ after rotation.

\textbf{The performance of DA under non-group transformations}
Most results of DA and ERM-type algorithms only hold when the transformation set is a group.
If we regard adversarial training as a special type of data augmentation through a ball around the natural data, then the transformation set is not a group.
The appropriate way to formulate theoretical guarantees for DA under arbitrary transformations is still an open question.

\section*{Acknowledgements}
This work was supported in part by the National Science Foundation under grant CCF-1815011 and by the Defense Advanced Research Projects Agency under cooperative agreement HR00112020003.
The views expressed in this work do not necessarily reflect the position or the policy of the Government and no official endorsement should be inferred. Approved for public release; distribution is unlimited.

We thank anonymous reviewers for their valuable suggestions.
HS thanks Freda Shi for discussion on the application of DA and suggestions from an applied viewpoint.

\bibliography{Major}
\section*{Checklist}
\begin{enumerate}

\item For all authors...
\begin{enumerate}
  \item Do the main claims made in the abstract and introduction accurately reflect the paper's contributions and scope?
    \answerYes{}
  \item Did you describe the limitations of your work?
    \answerYes{}
  \item Did you discuss any potential negative societal impacts of your work?
    \answerNA{This is a technical theory work.}
  \item Have you read the ethics review guidelines and ensured that your paper conforms to them?
    \answerYes{}
\end{enumerate}

\item If you are including theoretical results...
\begin{enumerate}
  \item Did you state the full set of assumptions of all theoretical results?
    \answerYes{}
        \item Did you include complete proofs of all theoretical results?
    \answerYes{}
\end{enumerate}

\item If you ran experiments...
\begin{enumerate}
  \item Did you include the code, data, and instructions needed to reproduce the main experimental results (either in the supplemental material or as a URL)?
    \answerNA{}
  \item Did you specify all the training details (e.g., data splits, hyperparameters, how they were chosen)?
    \answerNA{}
        \item Did you report error bars (e.g., with respect to the random seed after running experiments multiple times)?
    \answerNA{}
        \item Did you include the total amount of compute and the type of resources used (e.g., type of GPUs, internal cluster, or cloud provider)?
    \answerNA{}
\end{enumerate}

\item If you are using existing assets (e.g., code, data, models) or curating/releasing new assets...
\begin{enumerate}
  \item If your work uses existing assets, did you cite the creators?
    \answerNA{}
  \item Did you mention the license of the assets?
    \answerNA{}
  \item Did you include any new assets either in the supplemental material or as a URL?
    \answerNA{}
  \item Did you discuss whether and how consent was obtained from people whose data you're using/curating?
    \answerNA{}
  \item Did you discuss whether the data you are using/curating contains personally identifiable information or offensive content?
    \answerNA{}
\end{enumerate}

\item If you used crowdsourcing or conducted research with human subjects...
\begin{enumerate}
  \item Did you include the full text of instructions given to participants and screenshots, if applicable?
    \answerNA{}
  \item Did you describe any potential participant risks, with links to Institutional Review Board (IRB) approvals, if applicable?
    \answerNA{}
  \item Did you include the estimated hourly wage paid to participants and the total amount spent on participant compensation?
    \answerNA{}
\end{enumerate}

\end{enumerate}


\newpage
\appendix
\section{Proof of Theorem~\ref{thm:inv-da-ub}}\label{app:inv-da-ub}
\begin{proof}
    Let $d = \vcao(\cH,\cG)$. 
    According to the definition of DA and the invaraintly realizable setting, given the input $S\sim \cD^m$, the output of DA $\hat h\in \cH$ satisfies $\err_{\cG S}(\hat h) = 0$, i.e., $\hat h(x) = y$ for all $(x,y)\in \cG S$.
    Consider two sets $S$ and $S'$ of $m$ i.i.d. samples drawn from the data distribution $\cD$ each.
    We denote $A_{S}$ the event of $\{\exists h\in \cH, \err_\cD(h)\geq \epsilon, \err_{\cG S}(h)=0\}$ and $B_{S,S'}$ the event of $\{\exists h\in \cH, \err_{S'}(h)\geq \frac{\epsilon}{2},  \err_{\cG S}(h)=0\}$.
    By Chernoff bound, we have $\Pr(B_{S,S'})\geq \Pr(A_S)\cdot \Pr(B_{S,S'}|A_S)\geq \frac{1}{2}\Pr(A_S)$ when $m\geq \frac{8}{\epsilon}$.
    The sampling process of $S$ and $S'$ is equivalent to drawing $2m$ i.i.d. samples and then randomly partitioning into $S$ and $S'$ of $m$ each.
    For any fixed $S''$, for any $h$ with $\err_{S'}(h)\geq \frac{\epsilon}{2}$ and $\err_{\cG S}(h)=0$, if $h$ misclassifies some $(x,y)\in S''$, then all examples in the orbit of $x$, i.e., $\cG \{(x,y)\} \cap S''$, must go to $S'$. 
    Now to prove the theorem, we divide $S''$ into two categories in terms of the number of examples in each orbit. 
    Let $R_1 = \{x|\abs{\cG x \cap S''_{\cX}}\geq \log^2 m, x\in S''_\cX\}$ and $R_2 = S''_{\cX}\setminus R_1$. 
    For any $h$ making at least $\frac{\epsilon m}{2}$ mistakes in $S''_\cX$, $h$ either makes at least $\frac{\epsilon m}{4}$ mistakes in $R_1$ or makes at least $\frac{\epsilon m}{4}$ mistakes in $R_2$.
    Then let $\cH_0\subset \cH$ denote the set of hypotheses making at least $\frac{\epsilon m}{2}$ mistakes in $S''_\cX$ and divide $\cH_0$ into two sub-classes as follows.
    \begin{itemize}
        \item Let $\cH_1=\{h\in \cH_0| h \text{ makes at least } \frac{\epsilon m}{4} \text{ mistakes in } R_1\}$.
        For any $h\in \cH_1$, we let $X(h)\subset R_1$ denote a minimal set of examples in $R_1$ (breaking ties arbitrarily but in a fixed way) such that $h$ misclassify $X(h)$ and
        $\abs{(\cG X(h))\cap S''_\cX}\geq \frac{\epsilon m}{4}$ where $\cG X(h) = \{g x |x\in X(h), g\in \cG\}$ is the set of all examples lying in the orbits generated from $X(h)$. 
        Let $K(h) = \cG X(h)$ and $\cK = \{K(h)|h\in \cH_1\}$ the collection of all such sets.
        Notice that each example in $X(h)$ must belong to different orbits, otherwise it is not minimal. 
        Besides, each orbit in $K(h)$ contains at least $\log^2 m$ examples from $S''_\cX$ according to the definition of $R_1$.
        Hence, we have $\abs{X(h)}\leq \frac{\epsilon m}{4 \log^2 m}$.
        Since there are at most $\frac{2m}{\log^2m}$ orbits generated from $R_1$, we have $\abs{\cK} \leq \sum_{i=1}^{\frac{\epsilon m}{4 \log^2 m}} {\frac{2m}{\log^2 m}\choose i} \leq \left(\frac{8e}{\epsilon}\right)^{\frac{\epsilon m}{4 \log^2 m}}$. 
        Recall that $\err_{\cG S}(h)=0$ iff. $h(x)=y$ for all $(x,y) \in \cG S$.
        Since $h$ misclassify $X(h)$, all examples in their orbits must go to $S'$ to guarantee $\err_{\cG S}(h)=0$.
        Thus, we have
        \begin{align*}
            &\Pr(\exists h\in \cH_1, \err_{S'}(h)\geq \frac{\epsilon}{2},  \err_{\cG S}(h)=0)\\
            \leq &\Pr(\exists h\in \cH_1, K(h)\cap S_\cX=\emptyset)
            = \Pr(\exists K\in \cK, K\cap S_\cX=\emptyset)\\
            \leq &\sum_{K\in \cK} 2^{-\frac{\epsilon m}{4}} \leq \left(\frac{8e}{\epsilon}\right)^{\frac{\epsilon m}{4 \log^2 m}}\cdot 2^{-\frac{\epsilon m}{4}} = 2^{-\frac{\epsilon m}{4}(1-\frac{\log(8 e/\epsilon)}{\log^2 m})}\leq 2^{-\frac{\epsilon m}{8}}\,,
        \end{align*}
        when $m\geq \frac{8e}{\epsilon}+ 4$.

        \item Let $\cH_2 = \cH_0\setminus \cH_1$. 
        That is to say, for all $h\in \cH_2$, $h$ will make at least $\frac{\epsilon m}{4}$ mistakes in $R_2$. 
        Since $\vcao(\cH,\cG) =d$ and every orbit generated from $R_2$ contains fewer than $\log^2 m$ examples in $S''_\cX$, 
        the number of examples in $R_2$ that can be shattered by $\cH$ is no greater than $d\log^2 m$. 
        Thus the number of ways labeling examples in $R_2$ is upper bounded by $(\frac{2em}{d})^{d\log^2 m}$ by Sauer's lemma. 
        Hence, we have
        \begin{align*}
            \Pr(\exists h\in \cH_2, \err_{S'}(h)\geq \frac{\epsilon}{2},  \err_{\cG S}(h)=0)\leq (\frac{2em}{d})^{d\log^2 m}\cdot 2^{-\frac{\epsilon m}{4}} = 2^{-\frac{\epsilon m}{4}+ d\log^2 m \log(2em/d)}\,.
        \end{align*}
    \end{itemize}
    Combining the results for $\cH_1$ and $\cH_2$, we have
    \begin{align*}
        \Pr(B_{S,S'}) \leq &\Pr(\exists h\in \cH_1, \err_{S'}(h)\geq \frac{\epsilon}{2},  \err_{\cG S}(h)=0) + \Pr(\exists h\in \cH_2, \err_{S'}(h)\geq \frac{\epsilon}{2},  \err_{\cG S}(h)=0)\\
        \leq & 2^{-\frac{\epsilon m}{8}} + 2^{-\frac{\epsilon m}{4}+ d\log^2 m \log(2em/d)}\\
        \leq & \frac{\delta}{2},
    \end{align*}
    when $m\geq \frac{8}{\epsilon}(d\log^2 m\log \frac{2em}{d}+\log \frac 4 \delta + e) + 4$.
\end{proof}
\section{Proof of Theorem~\ref{thm:inv-da-lb}}\label{app:inv-da-lb}
\begin{proof}
    For any $d>0$, for any $\cX,\cH,\cG$ satisfying that there exists a subset $X=\{x_0,x_1,\ldots,x_{2d}\}\subset \cX$ such that
    \begin{itemize}
        \item their orbits are pairwise disjoint;
        \item for all $u\subset \{1,\ldots,2d\}$ with $\abs{u} =d$, there exists an $h_u\in \cH$ such that $h_u(g x_i)=1$ for all $g\in \cG, i\in u\cup\{0\}$ and $h_u(x_i) =0$ for $i\in \{0,\ldots,2d\}\setminus u$,
    \end{itemize}
    we will prove the theorem for $\cH' = \{h_u|u\subset \{1,\ldots,2d\}, \abs{u}=d\}$.
    For $\cX,\cH,\cG$ satisfying the above conditions, we have $\vcao(\cH',\cG)\geq d$ and $\vco(\cH',\cG)$ between $0$ and $\vcao(\cH',\cG)$.
    Then consider that the target function $h^*=h_{u^*}$ is chosen uniformly at random from $\cH'$. 
    The marginal data distribution $\cD_\cX$ puts probability mass $1-16\epsilon$ on $x_0$ and $\frac{16\epsilon}{d}$ on each point in $X_{u^*}:= \{x_i|i\in u^*\}$. 
    Then the target function $h^*$ is $(\cG, \cD_\cX)$-invariant. 
    
    Let the sample size $m = \frac{d}{64\epsilon}$.
    Given the training set $S_\trn \sim \cD^m$, the expected number of sampled examples in $X_{u^*}$ is $\frac{d}{4}$.
    By Markov's inequality, with probability greater than $1/2$, we observed fewer than $\frac{d}{2}$ points of $X_{u^*}$ in $S_\trn$ (denoted as event $B$). 
    Let $\cA$ be any proper learner, which means $\cA$ must output a hypothesis in $\cH'$.
    For any $h\in \cH'$ consistent with $\cG S_\trn$, $h$ must predict $d$ unobserved points in $\{x_1,\ldots,x_{2d}\}$ as $0$.
    Since for each unobserved point in $\{x_1,\ldots,x_{2d}\}$ labeled as $0$ by $h$, conditioned on event $B$, this point has probability greater than $\frac{1}{3}$ to be in $X_{u^*}$, which implies it is misclassified by $h$.
    By following the stardard technique, let $\err'(h) = \Pr_{(x,y)\sim \cD}(h(x)\neq y \wedge x\in X_{u^*})$, which is no greater than $\err(h)$ for any predictor $h$.
    Hence,
    \begin{align*}
        \EEs{h^*,S_\trn}{\err'(\cA(S_\trn))|B}
        =& 
        \EEs{S_\trn}{\EEs{h^*}{\err'(\cA(S_\trn))|S_\trn,B}|B}\\
        = &\EEs{S_\trn}{\EEs{h^*}{\frac{16\epsilon}{d}\sum_{i\in [2d]:\cA(S_\trn,x_i)=0}\ind{i\in u^*}|S_\trn,B}|B}\\
        = &\frac{16\epsilon}{d}\EEs{S_\trn}{\sum_{i\in [2d]:\cA(S_\trn,i)=0}\EEs{h^*}{\ind{i\in u^*}|S_\trn,B}|B}\\
        > &\frac{16\epsilon}{d}\EEs{S_\trn}{\sum_{i\in [2d]:\cA(S_\trn,i)=0}\frac{1}{3}|B}\\
        =&\frac{16\epsilon}{3}\,.
    \end{align*}
    Then we have
    \[\EEs{h^*,S_\trn}{\err'(\cA(S_\trn))}\geq \EEs{h^*,S_\trn}{\err'(\cA(S_\trn))|B} \cdot \Pr(B) > \frac{8\epsilon}{3}\,.\]
    Thus, for any proper learner $\cA$, there exists a target hypothesis $h^*\in \cH'$ and a data distribution s.t. $\EEs{S_\trn}{\err'(\cA(S_\trn)))}> 8\epsilon/3$. 
    Since $\err'(h)\leq 16\epsilon$ for any predictor $h$, with probability greater than $\frac{1}{9}$, 
    $\err(\cA(S_\trn))\geq \err'(\cA(S_\trn)) >\epsilon$.

    Here is an example of $\cX,\cH',\cG$ satisfying the above conditions.
    Let $\cX = \{0,\pm 1, \pm 2, \ldots, \pm 2d\}$. 
    The group $\cG$ is defined as $\cG_d=\{e, -e\}$ where $e$ is the identity element.
    Thus $\cX$ can be divided into $2d+1$ pairwise disjoint orbits, $\{\{0\}, \{\pm 1\},\ldots,\{\pm 2d\}\}$.
    For any $u\subset [2d]$ with $\abs{u}=d$, define $h_u := 1-\ind{[2d]\setminus u}$, which labels $[2d]\setminus u$ by $0$ and the other points by $1$. 
    Then we define the hypothesis class $\cH' = \{h_u|\abs{u}=d\}$.
    Since $\{1,\ldots,d\}$ can be shattered and no $d+1$ points can be shattered by $\cH'$,
    we have $\vcao(\cH',\cG)=d$.
    And since $\{0,-1,\ldots,-2d\}$ can only be labeled as $1$ by $\cH'$, we have $\vco(\cH',\cG)=0$. 
\end{proof}
\section{Proof of Theorem~\ref{thm:inv-da-help}}\label{app:inv-da-help}
\begin{proof}
    Let $d = \vcd(\cH)$ and $X=\{x_1,x_2,\ldots,x_d\}$ be a set of examples shattered by $\cH$.
    Let $\cS_{d-1}$ denote the permutation group acting on $d-1$ objects.
    Then for any partition $(A,[d-1]\setminus A)$ of $[d-1]$, we let $\cG(A) := \{\sigma\in \cS_{d-1}| 
    \forall i\in A, \sigma(i)\in A\}$.
    By acting $\cG(A)$ on $X$, then $X$ are partitioned into three orbits: $\{x_i|i\in A\}, \{x_i|i\notin A\}$ and $\{x_d\}$.
    
    For convenience, we first consider the case of the instance space being $X$.
    Since there are only three orbits and every labeling of $X$ is realized by $\cH$, $\vcao(\cH,\cG(A))= \vco(\cH,\cG(A)) =3$ for all $A\subset [d-1]$.
    Consider that we pick a set $A^*$ uniformly at random from $2^{[d-1]}$.
    Let $h^* = \ind{\{x_i\}_{i\in A^*}}$ and $\cG= \cG(A^*)$.
    The data distribution put probability mass $1-16\epsilon$ on $x_d$ and the remaining $16\epsilon$ uniformly over $\{x_i|i\in [d-1]\}$.
    Then for the training set size $m=\frac{d-1}{64\epsilon}$, with probability at least $1/2$, at most half of $\{x_1,x_2,\ldots,x_{d-1}\}$ is sampled in the training set $S_\trn$.
    For any algorithm $\cA$ not knowing $\cG$, $\cA$ will output a hypothesis $\hat h = \cA(S_\trn)$, which does not depend on $\cG$.
    For each unobserved point $x$ in $\{x_1,x_2,\ldots,x_{d-1}\}$, $\cA$ has probability $1/2$ to misclassify $x$.
    Following the standard technique, let $\err'(h) := \Pr_{(x,y)\sim \cD}(h(x)\neq y \wedge x\in \{x_1,x_2,\ldots,x_{d-1}\})$ and then we have
    \begin{align*}
        \EEs{A^*,S_\trn}{\err'(\cA(S_\trn))} \geq 2\epsilon\,, 
    \end{align*} 
    which implies $\cM_\inv(\epsilon,\delta;\cH,\cG,\cA)=\Omega(\frac{d}{\epsilon} +\frac{1}{\epsilon}\log \frac{1}{\delta})$ by applying the standard technique in proving a lower bound of sample complexity in standard PAC learning.

    In the case where the instance space not being $X$, we modify $\cG(A^*)$ a little by arranging all points in $\{x|h^*(x)=1\}\setminus X$ in one orbit and all points in $\{x|h^*(x)=0\}\setminus X$ in another orbit. 
    Then there are at most $5$ orbits, and $\vcao(\cH,\cG)= \vco(\cH,\cG) \leq 5$.
\end{proof}
\section{Proof of Theorem~\ref{thm:inv-opt}}\label{app:inv-opt}

\begin{proof}
For any $k\leq \vco(\cH,\cG)$, let $X_k = \{x_1,\ldots,x_k\}$ be a set shattered in the way defined in Definition~\ref{def:vco}.
Then $X_k$ can be shattered by $\cH(X_k)=\{h_{|X_k}|h\in \cH \text{ is } (\cG,X_k) \text{-invariant}\}$ and $\vcd(\cH(X_k)) = k$.
Since any data distribution $\cD$ with $\cD_\cX(X_k)=1$ is $\cG$-invariant,  
any lower bound on the sample complexity of PAC learning of $\cH(X_k)$ also lower bounds the sample complexity of invariantly realizable PAC learning of $\cH$.
Then the lower bound follows by standard arguments from~\cite{vapnik1974,blumer1989learnability,ehrenfeucht1989general}.

For the upper bound, recall that the algorithm $\cA$ is defined by letting $\cA(S)(x) = Q_{\cH(X_S\cup \{x\}), X_{S}\cup \{x\}}(S, x)$ if $\cH(X_S\cup \{x\})\neq \emptyset$ and predicting arbitrarily if $\cH(X_S\cup \{x\})= \emptyset$ in Section~\ref{subsec:inv-opt}.
Due to the invariantly-realizable setting, if $S$ and the test point $(x,y)$ are i.i.d. from the data distribution, $h^*_{|X_S\cup \{x\}}$ is in $\cH(X_S\cup \{x\})$ a.s. and then, $\cH(X_S\cup \{x\})$ is nonempty.
Following the analogous proof by~\cite{haussler1994predicting} for standard PAC learning, we have
\begin{align}
    \EEs{S\sim \cD^t}{\err(\cA(S))} =& \EEs{(x_i,y_i)_{i\in[t+1]}\sim \cD^{t+1}}{\ind{\A(\{x_{i},y_{i}\}_{i\in [t]},x_{t+1})\neq y_{t+1}}}\nonumber \\
    =& \frac{1}{(t+1)!} \sum_{\sigma\in \text{Sym}(t+1) } \EE{\ind{\A(\{x_{\sigma(i)},y_{\sigma(i)}\}_{i\in [t]},x_{\sigma(t+1)})\neq y_{\sigma(t+1)}}}\nonumber\\
    =& \EE{\frac{1}{(t+1)!} \sum_{\sigma\in \text{Sym}(t+1) } \ind{\A(\{x_{\sigma(i)},y_{\sigma(i)}\}_{i\in [t]},x_{\sigma(t+1)})\neq y_{\sigma(t+1)}}}\nonumber\\
    \leq &  \frac{\EE{\vcd(\cH(\{x_i|i\in [t+1]\}))}}{t+1}\label{eq:apply-lmm1}\\
    \leq & \frac{\vco(\cH,\cG)}{t+1}\label{eq:apply-defvco}\,,
\end{align}
where Eq~\eqref{eq:apply-lmm1} adopts Lemma~\ref{lmm:1-inclusion} and Eq~\eqref{eq:apply-defvco} holds due to the definition of $\vco(\cH,\cG)$.
To convert this algorithm, guaranteeing the expected error upper bounded by $\vco(\cH,\cG)$, into an algorithm with high probability $1-\delta$, we again follow an argument of~\cite{haussler1994predicting}.
Specifically, the algorithm runs $\cA$ for $\ceil{\log(2/\delta)}$ times, each time using a new sample of size $\ceil{4\vco(\cH,\cG)/\epsilon}$.
Then the algorithm selects the hypothesis from the outputs with the minimal error on a new sample of size $\ceil{32/\epsilon(\ln(2/\delta)+\ln(\ceil{\log(2/\delta)}+1))}$.
\end{proof}
\section{Proof of Theorem~\ref{thm:re-da}}\label{app:re-da}
We first introduce a useful lemma about a well-known Boosting algorithm, known as $\alpha$-Boost.
Given access to a weak learning algorithm, it can output a hypothesis with strong learning guarantee.
See~\cite{schapire2012boosting} for a proof.
\begin{lemma}[Boosting]\label{lmm:boosting}
    For any $k, n\in \NN$ and multiset $(x_1,y_1),\ldots,(x_n,y_n)\in \cX\times\cY$, suppose $\cA_0$ is an algorithm that, for any distribution $\cP$ on $\cX\times\cY$ with $\cP(\{(x_1,y_1),\ldots,(x_n,y_n)\})=1$, there exists $S_\cP\in \{(x_1,y_1),\ldots,(x_n,y_n)\}^k$ with $\err_\cP(\cA_0(S_\cP))\leq 1/3$. Then there is a numerical constant $c\geq 1$ such that, for $T=\ceil{c\log(n)}$, there exists multisets $S_1,\ldots,S_T\in \{(x_1,y_1),\ldots,(x_n,y_n)\}^k$ such that, for $\hat h(\cdot) = \Majority(\cA_0(S_1)(\cdot),\ldots,\cA_0(S_T)(\cdot))$, it holds that $\hat h(x_i)=y_i$ for all $i\in [n]$.
\end{lemma}
Part of the proof relies on a well-known generalization bound for compression schemes.
The following is the classic result due to \cite{littlestone1986relating}.

\begin{lemma}[Consistent compression generalization bound]\label{lmm:re-compression}
    There exists a finite numerical constant $c>0$ such that, for any compression scheme $(\kappa,\rho)$, for any $n\in \NN$ and $\delta\in (0,1)$, for any distribution $\cD$ on $\cX\times \cY$, for $S\sim \cD^n$,
    with probability at least $1-\delta$, if $\err_S(\rho(\kappa(S)))=0$, then
    \begin{equation*}
        \err_\cD(\rho(\kappa(S)))\leq \frac{c}{n-\abs{\kappa(S)}}(\abs{\kappa(S)} \log(n)+\log(1/\delta))\,.
    \end{equation*}
\end{lemma}

\begin{proof}[Proof of the first part of Theorem~\ref{thm:re-da}]
    The proof is inspired by the idea that representing algorithms by an orientation in a 1-inclusion graph in the transductive setting by~\cite{daniely2014optimal}.
    We will first prove a lower bound in the transductive setting and then extend the result to the inductive setting.
    For any $t\geq 2$, denote $\mu = \mu(\cH,\cG,t)$ and let $\bphi$ with $\abs{\bphi}=t$ and $P\in \Delta(B(\cH,\cG,\bphi))$ be given such that $\mu(\cH,\cG,\bphi,P)\geq \frac{\mu}{2}$.
    A augmented dataset $\cG S = \{(gx,y)|(x,y)\in S, g\in \cG\}$ is the same as a multiset of labeled orbits $\{\cG x\times \{y\}|(x,y)\in S\}$ up to different data formats. 
    For convenience, we overload the notation a little by also allowing $\cA$ being a mapping from a multiset of labeled orbits to a hypothesis.
    Then we construct a 1-inclusion graph $G_{\cH,\cG}(\bphi) = \{V,E\}$ as introduced in Section~\ref{sec:re} and define a mapping $w_\cA\in W$ as follows.
    For any edge $e = \{\bff,\bg,x_i\}\in E$, let
    \[w_\cA(\{\bff,\bg,x_i\},\bff) = \Pr(\cA(\bphi_{-i},\bff_{-i},x_i) = g_i)\,,\]
    and 
    \[w_\A(\{\bff,\bg,x_i\},\bg) = \Pr(\cA(\bphi_{-i},\bff_{-i},x_i) = f_i)\,,\]
    which is well-defined as $\Pr(\cA(\bphi_{-i},\bff_{-i},x_i) = g_i)+\Pr(\cA(\bphi_{-i},\bff_{-i},x_i) = f_i) =1$.
    Suppose our target function and the instance sequence $(\bff, \bx)$ is drawn from the distribution $P$.
    Then the expected number of mistakes in the transductive learning setting is
    \begin{align}
        &\EEs{(\bff,\bx)\sim P,\cA}{\sum_{i=1}^t \ind{\cA(\bphi_{-i},\bff_{-i},x_i)\neq f_i}}\nonumber\\
        \geq & \EEs{(\bff,\bx)\sim P}{\sum_{i\in [t]: \exists e\in E, \{\bff,x_i\}\subset e} w_\A(e(\bff,x_i),\bff)}\nonumber\\
        \geq &\min_{w}\EEs{(\bff,\bx)\sim P}{\sum_{i\in [t]: \exists e\in E, \{\bff,x_i\}\subset e} w(e(\bff,x_i),\bff)}\nonumber\\
        = & \mu(\cH,\cG,\bphi,P)\,,\label{eq:lb-mu-trans}
    \end{align}
    where Eq~\eqref{eq:lb-mu-trans} holds due to the definition of $\mu(\cH,\cG,\bphi,P)$.
    
    Now we prove the lower bound in the inductive setting based on the similar idea. 
    Let $\bphi$ and $P$ be the same as those in the transductive setting.
    For any $\epsilon <\frac{\mu}{16(t-1)}$, we draw $(\bff,\bx)\sim P$ and then let $\bff$ be our target function and let the marginal data distribution be like,
    putting probability mass $1-\frac{16(t-1)\epsilon}{\mu}$ on $x_t$ and the remaining probability mass uniformly over  $\{x_1,\ldots,x_{t-1}\}$.
    Denote this data distribution by $\cD_{\bff,\bx}$.
    For any fixed $i\in [t-1]$, $\Pr(x_i \notin S_{\trn,\cX}) = (1-\frac{16\epsilon}{\mu})^m\geq (\frac{1}{4})^{\frac{16m\epsilon}{\mu}}\geq \frac{1}{2}$ when $m\leq \frac{\mu}{32\epsilon}$.
    For any hypothesis $h$, let $\err'_\cD(h) = \Pr_{(x,y)\sim \cD}(h(x)\neq y \text{ and } x\in \{x_i\}_{i\in [t-1]})$ and we always have $\err(h)\geq \err'(h)$. Then we have
    \begin{align}
        &\EEs{(\bff,\bx)\sim P, S_\trn\sim \cD_{\bff,\bx}^m, \cA}{\err'(\cA(\cG S_\trn))}\nonumber\\
        = & \frac{16\epsilon}{\mu}\sum_{i=1}^{t-1}\EEs{(\bff,\bx)\sim P, S_\trn\sim \cD_{\bff,\bx}^m, \cA}{\ind{\cA(\cG S_\trn,x_i)\neq f_i}}\nonumber\\
        \geq& \frac{16\epsilon}{\mu}\sum_{i=1}^{t-1}\EEs{(\bff,\bx)\sim P}{\Pr_{S_\trn\sim \cD_{\bff,\bx}^m, \cA}(\cA(\cG S_\trn,x_i)\neq f_i|x_i\notin S_{\trn,\cX})\Pr(x_i\notin S_{\trn,\cX})}\nonumber\\
        \geq & \frac{8\epsilon}{\mu}\EEs{(\bff,\bx)\sim P}{\sum_{i=1}^{t-1}\Pr_{S_\trn\sim \cD_{\bff,\bx}^m, \cA}(\cA(\cG S_\trn,x_i)\neq f_i|x_i\notin S_{\trn,\cX})}\label{eq:eq:lb-mu-indc}\,.
    \end{align}
    For all $i\in[t-1]$, for all $z\in \phi_i$, if there is an edge $e=\{\bff,\bg, z\}\in E$, we let
    \begin{equation*}
        \tilde w(e,\bff) = \Pr_{S_\trn\sim \cD_{\bff,\bx}^m, \cA}(\cA(\cG S_\trn,z)\neq f_i|z\notin S_{\trn,\cX})\,,
    \end{equation*}
    where $\bx$ is an arbitrary sequence in $\cU_{\bff}(\bphi)$ satisfying that $x_i = z$.
    Then $\tilde w(e,\bff)$ is well-defined since the distribution of $\cG S_\trn$ conditioned on $x_i\notin S_{\trn,\cX}$ is the same for all $\bx\in \cU_{\bff}(\bphi)$ with $x_i = z$.
    Actually, conditioned on $x_i\notin S_{\trn,\cX}$, the distribution of $\cG S_\trn$ is also the same when $S_\trn$ is sampled from $\cD_{\bg,\bx'}$ where $\bx'$ is an arbitrary sequence in  $\cU_{\bg}(\bphi)$ satisfying that $x_i' = z$.
    Hence, $\tilde w(e,\bff) + \tilde w(e,\bg)=1$.
    By letting $\tilde w(e,\bh)=0$ for all $\bh\notin e$, $\tilde w$ is in $W$.
    Then we have 
    \begin{align*}
        \text{Eq~\eqref{eq:eq:lb-mu-indc}} &\geq \frac{8\epsilon}{\mu}\EEs{(\bff,\bx)\sim P}{\sum_{i\in [t-1]: \exists e\in E, \{\bff,x_i\}\subset e}\tilde w(e,\bff)}\\
        &\geq \frac{8\epsilon}{\mu}\min_{w\in W}\EEs{(\bff,\bx)\sim P}{\sum_{i\in [t]: \exists e\in E, \{\bff,x_i\}\subset e} w(e,\bff)-1}\\
        &\geq 4\epsilon(1-2/\mu)\geq 2\epsilon,
    \end{align*}
    when $\mu\geq 4$. Hence, there exists a labeling function $\bff$ (i.e., there exists a target function $h^*$) and a data distribution $\cD_{\bff,\bx}$ such that $\EEs{S_\trn,\cA}{\err'(\cA(\cG S_\trn))}\geq 2\epsilon$. Since $\err'(h)\leq \frac{16(t-1)\epsilon}{\mu}$ for all hypothesis $h$, $\Pr(\err(\cA(\cG S_\trn)) >\epsilon)\geq \Pr(\err'(\cA(\cG S_\trn)) >\epsilon)>\frac{\mu}{16(t-1)}$.
\end{proof}

\begin{proof}[Proof of the second and third parts of Theorem~\ref{thm:re-da}]
    For any $n\in \NN$, for any given sample $\{(x_1,y_1),\ldots,(x_{n+1},y_{n+1})\}$, let $\bphi =\{\phi_1,\ldots,\phi_{n+1}\} =\{\cG x_1,\ldots, \cG x_{n+1}\}$ denote the multi-set of $t+1$ orbits and construct the one-inclusion graph $G_{\cH,\cG}(\bphi)$.
    As mentioned in Definition~\ref{def:mu}, every $w\in W$ defines a randomized orientation of each edge in graph $G_{\cH,\cG}(\bphi)$.
    That is, for any fixed $w\in W$, for every edge $e=\{\bff,\bg,x\}$, $w$ defines a probability over $\{\bff, \bg\}$.
    Then we can construct an algorithm $\A_w$ for each $w\in W$.
    
    Given the input $(\bphi_{-i},\vec y_{-i}, x_i)$, $\A_w$ finds the subset of vertices $\{(\bff,x_i)|\bff_{-i} = \vec y_{-i}\}$ whose labelings are consistent with $(\bphi_{-i},\vec y_{-i})$. 
    If there exist two such vertices, $(\bff,x_i)$ and $(\bg,x_i)$, they must be connected by $e=(\bff,\bg,x_i)$ due to the definition of $G_{\cH,\cG}(\bphi)$.
    Then $\A_w$ will predict $x_i$ as $f_i$ with probability $w(e,\bg)$ and as $g_i$ with probability $w(e,\bff)$.
    If only one such vertex $(\bff,x_i)$ exists, $\A_w$  predicts the label of $x_i$ by $f_i$. 
    Due to the realizable setting, there must exist at least one such vertex and the algorithm's prediction must be correct when only one vertex exists.
    
    To complete the algorithm, the remaining part is how to choose a good $w\in W$. 
    For any true labeling $\bff^*$ and any sequence of natural data $\bx^*\in \cU_{\bff^*}(\bphi)$, for each $i\in [n+1]$, if there is an edge $e\supset \{\bff^*, x_i^*\}$, it means the algorithm possibly misclassify $x_i^*$ (with the probability dependent on $w$); if there is no such an edge, it means the algorithm will not misclassify $x_i^*$ no matter what $w$ is. 
    For any labeling $\bff\in \Pi_\cH(\bphi)$ and a sequence of natural data $\bx\in \cU_{f}(\bphi)$, 
    we can represent the subset of the points in $\bx$ that the algorithm is uncertain about by a mapping $a_{\bff,\bx}: {E\times \Pi_\cH(\bphi)}\mapsto \{0,1\}$ where $a_{\bff,\bx}(e,\bg) = 1$ iff. $\bg=\bff$ and there exists $i\in [n+1]$ s.t. $\{\bff,x_i\}\in e$.
    Due to the definition, $a_{\bff,\bx}$ has at most $n+1$ non-zero entries.
    Let $A = \{a_{\bff,\bx}|\bff\in \Pi_\cH(\bphi), \bx\in \cU_{f}(\bphi)\}$ denote the set of all such mappings. 
    We now first consider the case where $\abs{E}<\infty$.
    Then for a training set of size $n$, we can rewrite the expected error as
    \begin{align}
        &\EEs{S\sim \cD^n}{\err(\A_w(\cG S))}\nonumber\\
        =&\EEs{S\sim \cD^n, (x,y)\sim \cD,\A_w}{\ind{\A_w(\cG S, x)\neq y}}\nonumber\\
        =&\EEs{(\bx,\vec y)\sim \cD^{n+1},\A_w}{\frac{1}{n+1}\sum_{i=1}^{n+1}\ind{\A_w(\bphi_{-i}, \vec y_{-i}, x_i)\neq y_i}}\nonumber\\
        =&\EEs{(\bx,\vec y)\sim \cD^{n+1},\A_w}{\frac{1}{n+1}\sum_{i=1}^{n+1}\ind{\A_w(\bphi_{-i},\vec y_{-i}, x_i)\neq y_i}\ind{\exists e\supset \{x_i,\bff^*\}}}\nonumber\\
        =&\EEs{(\bx,\vec y)\sim \cD^{n+1}}{\frac{1}{n+1}\sum_{e\in E}w(e,\bff^*)a_{\bff^*,\bx}(e,\bff^*)}\nonumber\\
        =&\EEs{(\bx,\vec y)\sim \cD^{n+1}}{\frac{1}{n+1}\sum_{e\in E, \bff\in \Pi_\cH(\bphi)}w(e,\bff)a_{\bff^*,\bx}(e,\bff)}\,.\label{eq:err-ub-harm}
    \end{align}
    where the last equality holds due to $a_{\bff^*,\bx}(e,\bff)=0$ for all $\bff\neq \bff^*$.
    Since $\bff^*$ and $\bx$ is unknown, our goal is to find a $w$ with $\sum_{e\in E, \bff\in \Pi_\cH(\bphi)}w(e,\bff)a_{\bff^*,\bx}(e,\bff)$ upper bounded for all $\bff^*$ and $\bx$. 
    The algorithm picks $w^* = \argmin_{w\in W}\max_{a\in A} \sum_{e\in E, \bff\in \Pi_{\cH}(\bphi)}w(e,\bff)a(e,\bff)$. 
    Then we have
    \begin{align}
        &\max_{a\in A}\sum_{e\in E, \bff\in \Pi_\cH(\bphi)}w^*(e,\bff)a(e,\bff)= \min_{w\in W} \max_{a\in A} \sum_{e\in E, \bff\in \Pi_\cH(\bphi)}w(e,\bff)a(e,\bff)\nonumber\\
        = & \min_{w\in W} \max_{a\in \conv(A)} \sum_{e\in E, \bff\in \Pi_\cH(\bphi)}w(e,\bff)a(e,\bff)= \max_{a\in \conv(A)}\min_{w\in W} \sum_{e\in E, \bff\in \Pi_\cH(\bphi)}w(e,\bff)a(e,\bff)\label{eq:finite-minimax}\,,
    \end{align}
    where the last equality is due to Minimax theorem.
    Since the optimal solution $a^*$ to Eq.~\eqref{eq:finite-minimax} is in the convex hull of $A$, 
    there is a distribution $P^*\in \Delta(B(\cH,\cG,\bphi))$ such that $a^* = \EEs{(\bff^*,\bx)\sim P^*}{a_{\bff^*,\bx}}$.
    Then we have
    \begin{align*}
        \text{Eq~\eqref{eq:finite-minimax}} = &\min_{w\in W}\sum_{e\in E, \bff\in \Pi_\cH(\bphi)}w(e,\bff) \EEs{(\bff^*,\bx)\sim P^*}{a_{\bff^*,\bx}(e,\bff)}\\
        = &\min_{w\in W}\EEs{(\bff^*,\bx)\sim P^*}{\sum_{e\in E, \bff\in \Pi_\cH(\bphi)}w(e,\bff) a_{\bff^*,\bx}(e,\bff)}\\
        = &\min_{w\in W}\EEs{(\bff^*,\bx)\sim P^*}{\sum_{i\in [n+1]: \exists e\in E, \{\bff^*,x_i\}\subset e} w(e(\bff^*,x_i),\bff^*)}\\
        =& \mu(\cH,\cG,\bphi, P^*)\,.
    \end{align*}
    Combined with Eq~\eqref{eq:err-ub-harm}, we have $\EEs{S \sim \cD^n}{\err(\A_{w^*}(\cG S))}\leq \frac{\mu(\cH,\cG,\bphi,n+1)}{n+1}$.
    
    For $\abs{E}=\infty$, $E\times \Pi_\cH(\bphi)$ could be infinite dimensional and thus, we need to use Sion's minimax theorem.
    The details of how to apply Sion's minimax theorem are described as follows.
    For all $w\in W$, for any edge $e = \{\bff,\bg, x\}$, we have $w(e,\bff) + w(e,\bg) =1$ and thus,
    there exists a one-to-one mapping $\beta: W\mapsto [0,1]^E$ where $\beta(w)(e) = w(e,\bff_{e, 0})$ where $\bff_{e, 0}$ is the labeling in $e$ predicting $x$ as zero.
    In the following, we will overload the notation by using $w$ to represent $\beta(w)$ when it is clear from the context that it is in the space $[0,1]^E$.
    Then we define a mapping $\BL(\cdot,\cdot): [0,1]^E\times A$ by 
    \begin{equation*}
        \BL(w,a):=\sum_{e:a(e,\bff_{e,0})=1}w(e) +\sum_{e:a(e,\bff_{e,1})=1}(1-w(e))\,,
    \end{equation*}
    where $\bff_{e, 0}, \bff_{e, 1}$ are labelings in $e = \{\bff_{e, 0}, \bff_{e, 1},x\}$ and they label $x$ as $0$ and $1$ respectively.
    Then similar to Eq~\eqref{eq:err-ub-harm}, we can represent the expected error as
    \begin{align*}
        \EEs{S \sim \cD^n}{\err(\A_w (\cG S))} = \EEs{(\bx,\vec y)\sim \cD^{n+1}}{\frac{1}{n+1}\BL(w,a_{\bff^*,\bx})}\,.
    \end{align*}
    Now we want to upper bound $\min_{w\in [0,1]^{E}} \sup_{a\in A}\BL(w,a)$.
    $\BL(\cdot,\cdot)$ can be extended to $\R^E\times A$ by letting $\BL(w,a)=\sum_{e:a(e,\bff_{e,0})=1}w(e) +\sum_{e:a(e,\bff_{e,1})=1}(1-w(e))$ for $w\in \R^E$.
    Since $a$ has at most $n+1$ non-zeros entries, $\abs{\BL(w,a)}\leq (n+1)\max(\norm{w}_\infty, \norm{\bOne - w}_\infty)$.
    Let $\tilde A=\{\sum_{i=1}^N c_i a_i|\forall i, a_i\in A, c_i>0,\sum_{i=1}^N c_i =1, N\in \NN\}$ be the set of all finite convex combination of elements in $A$.
    We extend $\BL$ from $R^{E} \times A$ to $R^{E}\times \tilde A$ by defining $\BL(w,a) = \sum_{i=1}^N c_i \BL(w,a_i)$ for $a=\sum_{i=1}^N c_i a_i\in \tilde A$.

    We define a metric $d$ in $\tilde A$: for $a=\sum_{i=1}^N c(a_i) a_i$ and $a'  =\sum_{i=1}^{N'} c'(a_i') a_i'$, the distance between $a$ and $a'$ is $d(a,a') := \sum_{\alpha \in A:c(\alpha)\neq 0 \text{ or } c'(\alpha)\neq 0}\abs{c(\alpha) - c'(\alpha)}$.
    For any fixed $w\in \R^E$, for any $a\in \tilde A$,
    for every open ball $\cB_r(\BL(w,a))$ centered at $\BL(w,a)$ with radius $r>0$, 
    there is an open ball $\cB_{r'}(a)$ in the metric space $(\tilde A,d)$ with $r' = \frac{r}{(n+1)(\max(\norm{w}_{\infty}, \norm{\bOne - w}_{\infty}))}$ such that for all $a'\in \cB_{r'}(a)$,
    $\abs{\BL(w,a')-\BL(w,a)} \leq \sum_{\alpha \in \{a_i\}_{i\in [N]}\cup \{a_i'\}_{i\in [N']}}\abs{c(\alpha)-c'(\alpha)}\abs{\BL(w,\alpha)}< r'\cdot \abs{\BL(w,\alpha)}\leq r$. 
    Hence, $\BL(w, \cdot)$ is continuous for all $w\in \R^E$.

    Consider the the standard topology in $\R$ and then $[0,1]$ is compact.
    Then let $\cT$ be the product topology of $\R^E$.
    Then by Tychonoff theorem, $[0,1]^E$ is compact in $\R^E$.
    For any fixed $a = \sum_{i=1}^N c_i a_i \in \tilde A$, there are at most $N(n+1)$ non-zero entries. 
    Then for any $w\in \R^E$,
    for every open ball $\cB_r(\BL(w,a))$ centered at $\BL(w,a)$ with radius $r>0$, 
    then there is a neighborhood $U= \prod_{e\in E} S_e$, where $S_e = (w_e-\frac{r}{N(n+1)},w_e+\frac{r}{N(n+1)})$ if at least one of $a(e,\bff_{e,0}), a(e,\bff_{e,1})$ is non-zero and $S_e = \R$ for other $e\in E$, such that $\abs{\BL(w',a)- \BL(w,a)}<r$ for all $w'\in U$. 
    That is, $\BL(U)\subset \cB_r(\BL(w,a))$.
    Hence, $\BL(\cdot, a)$ is continuous for all $a\in \tilde A$.

    It is easy to check that $\BL(\cdot, a)$ and $\BL(w, \cdot)$ are linear for all $a\in \tilde A$, $w\in W$.
    Then by Sion's minimax theorem, we have 
    \begin{align*}
        \min_{w\in [0,1]^E} \sup_{a\in A}\BL(w,a)
        \leq  \min_{w\in [0,1]^E} \sup_{a\in \tilde A}\BL(w,a)
        = \sup_{a\in \tilde A}\min_{w\in [0,1]^E}\BL(w,a)=:v^*\,.
    \end{align*}
    Let $w^*=\argmin_{w\in [0,1]^E} \sup_{a\in A}\BL(w,a)$.
    There exists a sequence $a_1,a_2,\ldots$ in $\tilde A$ such that $\lim_{k\rightarrow\infty} \min_{w\in [0,1]^E}\BL(w,a_k) = v^*$.
    For each $a_k$, we let $P_k\in \Delta(B(\cH,\cG,\bphi))$ be the distribution over $B(\cH,\cG,\bphi)$ such that $a_k = \EEs{(\bff^*,\bx)\sim P_k}{a_{\bff^*,\bx}}$.
    Due to the definition of $\tilde A$, we know that $P_k$ is a discrete distribution with finite support.
    Then we have
    \begin{align*}
        \sup_{a\in A}\BL(w^*,a)\leq &\min_{w\in [0,1]^E} \sup_{a\in \tilde A}\BL(w,a)\\
        = &\sup_{a\in \tilde A}\min_{w\in [0,1]^E}\BL(w,a)\\
        = &\lim_{k\rightarrow\infty} \min_{w\in [0,1]^E}\BL(w,\EEs{(\bff^*,\bx)\sim P_k}{a_{\bff^*,\bx}})\\
        = &\lim_{k\rightarrow\infty} \min_{w\in [0,1]^E}\EEs{(\bff^*,\bx)\sim P_k}{\BL(w,a_{\bff^*,\bx})}\\
        \leq &\sup_{P\in \Delta(B(\cH,\cG,\bphi))}\min_{w\in [0,1]^E}\EEs{(\bff^*,\bx)\sim P}{\BL(w,a_{\bff^*,\bx})}\\
        = & \mu(\cH,\cG,\bphi,n+1)\,.
    \end{align*}
    Hence, $\EEs{S \sim \cD^n}{\err(\A_{w^*}(\cG S))}\leq \frac{\mu(\cH,\cG,\bphi,n+1)}{n+1}$.
    \paragraph{The first upper bound}  We can convert the above bound into a high probability bound by $\alpha$-Boost.
    Let $\cA_0 = \cA_{w^*}$ as defined above.
    Let $t$ be any positive integer such that $\frac{1}{6}\leq \frac{\mu(\cH,\cG,t)}{t}\leq \frac{1}{3}$.
    As established above, for $S_\trn = \{(x_1,y_1),\ldots,(x_m,y_m)\}\sim \cD^m$ and any distribution $\cP$ supported on $\{(x_1,y_1),\ldots,(x_m,y_m)\}$,
    for $k=t-1$ and $S\sim \cP^k$, $\EE{\err_{\cP}(\cA_0(S))}\leq 1/3$.
    Thus given $S_\trn$ and $\cP$, there exists a deterministic choice of $S_\cP\in \{(x_1,y_1),\ldots,(x_m,y_m)\}^k$ with $\err_{\cP}(\cA_0(S_\cP))\leq 1/3$.
    Then Lemma~\ref{lmm:boosting} implies that for a value $T=\ceil{c_1 \log m}$ (for numerical constant $c_1\geq 1$), 
    there exists $S_1,\ldots,S_T\in \{(x_1,y_1),\ldots,(x_m,y_m)\}^k$ such that, for $\hat h(\cdot) = \Majority(\cA_0(S_1)(\cdot),\ldots,\cA_0(S_T)(\cdot))$, 
    it holds that $\hat h(x_i)=y_i$ for all $i\in [m]$.
    Note that $\hat h$ can be expressed as a compression scheme.
    By Lemma~\ref{lmm:re-compression},
    with probability at least $1-\delta$,
    \[\err(\hat h)\leq \frac{c_2}{m-kT}\left(kT\log m + \log \frac{1}{\delta}\right)\,,\]
    for a numerical constant $c_2\geq 1$. 
    Thus, for any given $\epsilon\in (0,1)$, the right hand side can be made less than $\epsilon$ for an appropriate choice of
    \[m = O( \frac{1}{\epsilon} (k\log^2\frac{k}{\epsilon} + \log \frac{1}{\delta}))\,,\]
    where $k = t-1\leq 6\mu(\cH,\cG,t) -1$.
    \paragraph{The second upper bound} Again, using the same standard technique as we used in Theorem~\ref{thm:inv-opt} to convert an algorithm with expected error upper bound to an algorithm with high probability guarantee. 
    The algorithm runs $\cA_{w^*}$ for $\ceil{\log(2/\delta)}$ times, each time using a new sample of size $\ceil{4\mu(\cH,\cG)/\epsilon}$.
    Then the algorithm selects the hypothesis from the outputs with the minimal error on a new sample of size $\ceil{32/\epsilon(\ln(2/\delta)+\ln(\ceil{\log(2/\delta)}+1))}$.
\end{proof}
\section{Proof of Theorem~\ref{thm:re-mu-dim}}\label{app:re-mu-dim}
\begin{proof}
    Let $d = \dim(\cH,\cG)$.
    Let $\bphi$ and the corresponding $B = \{(\bff,\bx_\bff)\}_{\bff\in \cY^d}$ be given.
    Let $P$ be the uniform distribution over $B$.
    Then
    \begin{align}
        &\mu(\cH,\cG,\bphi,P)=\min_{w\in W} \EEs{(\bff,\bx_\bff)\sim P}{\sum_{i\in [d]:\exists e\in E, \{\bff,x_{\bff, i}\}\subset e} w(e,\bff)}\nonumber \\
        =& \min_w  \frac{1}{2^d}\sum_{\bff\in \cY^d}\sum_{i\in [d]:\exists e\in E, \{\bff,x_{\bff, i}\}\subset e} w(e,\bff)\nonumber\\
        =& \frac{1}{2^d} \min_w \sum_{\bff\in \cY^d}\sum_{i=1}^d \frac{1}{2}\left(w(e(\bff, x_{\bff, i}),\bff) + w(e(\bff, x_{\bff, i}),(1-f_i,\bff_{-i}))\right)\label{eq:apply-def-dim}\\
        =&\frac{1}{2^d}\cdot \frac{2^d\cdot d}{2}= \frac{d}{2}\nonumber\,,
    \end{align}
    where Eq~\eqref{eq:apply-def-dim} holds since $x_{\bff,i} = x_{\bg,i}$ if $\bff \oplus \bg = \be_i$ due to the definition of $\dim(\cH,\cG)$.
\end{proof}
\section{Proof of Theorem~\ref{thm:re-opt}}\label{app:re-opt}
\begin{proof}[Proof of the lower bound] 
    Let $d = \vcao(\cH,\cG)$, let $X_d = \{x_1,\ldots,x_d\}$ be a set shattered in the way defined in Definition~\ref{def:vcao} and let $\cH':= \{h_{|X_d}|h\in \cH\}$.
    Then $\vcd(\cH') = d$ and any lower bound on the sample complexity of PAC learning of $\cH'$ is also a lower bound on the sample complexity of relaxed realizable PAC learning of $\cH$.
    The lower bound follow by standard arguments from~\cite{vapnik1974,blumer1989learnability,ehrenfeucht1989general}.
\end{proof}

\begin{proof}[Proof of the upper bound of the algorithm ERM-INV]
    The algorithm ERM-INV works as follows.
    ERM-INV first applies ERM over the original data set.
    Then for every test instance $x$, 
    if $x$ lies in the orbits generated by the original data, 
    the algorithm predicts $x$ by the label of the training instance in the same orbit; 
    otherwise, the algorithm predicts according to the ERM output.
    Specifically, given the training set $S_\trn=\{(x_1,y_1),\ldots,(x_m,y_m)\}\sim \cD^m$, the algorithm finds a hypothesis $h\in \cH$ consistent with $S_\trn$ and then outputs $f_{h,S_\trn}$ defined by
    \begin{align}
        f_{h,S_\trn}(x) = 
        \begin{cases}
        y_i& \text{if there exists }i\in[m] \st x\in \cG x_i\,,\\
        h(x)& \text{o.w.}
        \end{cases}\label{eq:erm-inv}
    \end{align}
    The function $f_{h,S_\trn}$ is well-defined a.s. when the data distribution $\cD$ is $\cG$-invariant since if there exists $i\neq j$ such that $x_i \in \cG x_j$, then $y_i = y_j = h^*(x_j)$.

    The proof idea is similar to that of Theorem~\ref{thm:inv-da-ub}.
    Let $d = \vcao(\cH,\cG)$.
    Consider two sets $S$ and $S'$ of $m$ i.i.d. samples drawn from the data distribution $\cD$ each.
    We denote $A_{S}$ the event of $\{\exists h\in \cH, \err_{S}(h)=0, \err_\cD(f_{h,S})\geq \epsilon\}$ and $B_{S,S'}=\{\exists h\in \cH, \err_{S }(h)=0, \err_{S'}(f_{h,S})\geq \frac{\epsilon}{2}\}$.
    By Chernoff bound, we have $\Pr(B_{S,S'})\geq \Pr(A_S)\cdot \Pr(B_{S,S'}|A_S)\geq \frac{1}{2}\Pr(A_S)$ when $m\geq \frac{8}{\epsilon}$.
    The sampling process of $S$ and $S'$ is equivalent to drawing $2m$ i.i.d. samples and then randomly partitioning into $S$ and $S'$ of $m$ each.
    For any fixed $S''$, let us divide $S''$ into two categories in terms of the number of examples in each orbit.
    Let $R_1 = \{x|\abs{\cG x \cap S''_{\cX}}\geq \log^2 m\}$ and $R_2 = S''_{\cX}\setminus R_1$.
    Let $\cH_0\subset \cH$ denote the set hypotheses making at least $\frac{\epsilon m}{2}$ mistakes in $S''$.
    Now we divide $\cH_0$ into two sub-classes as follows.
    \begin{itemize}
        \item  Let $\cH_1=\{h\in \cH_0| h \text{ makes fewer than }\frac{\epsilon m}{4} \text{ mistakes in } R_2\}$.
        Let $\cF_{S''}(\cH_1) := \{f_{h,T}| h\in \cH_1, T\subset S'', \abs{T}=m\}$. 
        For any $h$, if $S$ is correctly labeled, then $\err_{\cG S}(f_{h,S})=0$.
        Thus we have
        \begin{align*}
            &\Pr(\exists h\in \cH_1, \err_{S}(h)=0, \err_{S'}(f_{h,S})\geq \frac{\epsilon}{2})\\
            \leq &\Pr(\exists h\in \cH_1, \err_{\cG S}(f_{h,S})=0, \err_{S'}(f_{h,S})\geq \frac{\epsilon}{2})\\
            \leq &\Pr(\exists f\in \cF_{S''}(\cH_1), \err_{\cG S}(f)=0, \err_{S'}(f)\geq \frac{\epsilon}{2})\,.
        \end{align*}
        Since every $h\in \cH_1$ makes fewer than $\frac{\epsilon m}{4}$ mistakes in $R_2$, for all $f\in \cF_{S''}(\cH_1)$, 
        if $f$ makes at least $\frac{\epsilon m}{2}$ mistakes in $S''$, 
        it must makes at least $\frac{\epsilon m}{4}$ mistakes in $R_1$. 
        Similar to the case~1 in the proof of Theorem~\ref{thm:inv-da-ub}, for any $f\in \cF_{S''}(\cH_1)$, 
        we let $X(f)\subset R_1$ denote a minimal set of examples in $R_1$ (breaking ties arbitrarily but in a fixed way) such that $f$ misclassify $X(f)$ and $\abs{(\cG X(f))\cap S''_\cX}\geq \frac{\epsilon m}{4}$ where $\cG X(f) = \{g x |x\in X(f), g\in \cG\}$ is the set of all examples lying in the orbits generated from $X(f)$. 
        Let $K(f) = \cG X(f)$ and $\cK = \{K(f)|f\in \cF_{S''}(\cH_1)\}$.
        Notice that each example in $X(f)$ must belong to different orbits, otherwise it is not minimal. 
        Besides, each orbit contains at least $\log^2 m$ examples from $S''_\cX$.
        Hence, $\abs{X(f)}\leq \frac{\epsilon m}{4 \log^2 m}$. 
        Since there are at most $\frac{2m}{\log^2 m}$ orbits generated from $R_1$, we have $\abs{\cK} \leq \sum_{i=1}^{\frac{\epsilon m}{4 \log^2 m}} {\frac{2m}{\log^2 m}\choose i} \leq \left(\frac{8e}{\epsilon}\right)^{\frac{\epsilon m}{4 \log^2 m}}$.
        Since $f$ misclassify $X(f)$, all examples in $K(f)$ must go to $S'$ to guarantee $\err_{\cG S}(f)=0$ and thus, 
        we have
        \begin{align*}
            &\Pr(\exists f\in \cF_{S''}(\cH_1), \err_{\cG S}(f)=0, \err_{S'}(f)\geq \frac{\epsilon}{2})\\
            \leq & \Pr(\exists f\in \cF_{S''}(\cH_1), K(f)\cap S_\cX=\emptyset)=\Pr(\exists K\in \cK, K\cap S_\cX=\emptyset)\\
            \leq &\sum_{K\in \cK} 2^{-\frac{\epsilon m}{4}} \leq \left(\frac{8e}{\epsilon}\right)^{\frac{\epsilon m}{4 \log^2 m}}\cdot 2^{-\frac{\epsilon m}{4}} = 2^{-\frac{\epsilon m}{4}(1-\frac{\log(8 e/\epsilon)}{\log^2 m})}\leq 2^{-\frac{\epsilon m}{8}}\,,
        \end{align*}
        when $m\geq \frac{8e}{\epsilon}+4$.
        
        \item Let $\cH_2 = \cH_0\setminus \cH_1$.
        Now we will bound $\Pr(\exists h\in \cH_2, \err_{S}(h)=0, \err_{S'}(f_{h,S})\geq \frac{\epsilon}{2})$.
        Similar to the case~2 in Theorem~\ref{thm:inv-da-ub}, every $h\in \cH_2$ will make at least $\frac{\epsilon m}{4}$ mistakes in $R_2$. 
        Since $\vcao(\cH,\cG) = d$ and every orbit generated from $R_2$ contains fewer than $\log^2 m$ examples, the number of examples in $R_2$ that can be shattered by $\cH$ is no greater than $d\log^2 m$. 
        Thus, the number of ways labeling examples in $R_2$ is upper bounded by $(\frac{2em}{d})^{d\log^2 m}$ by Sauer's lemma. 
        For any multi-subset $X\subset S''_\cX$ and hypothesis $h$, we denote by $\hat M_{X}(h)$ the number of instances in $X$ misclassified by $h$.
        Hence, we have
    \begin{align*}
        &\Pr(\exists h\in \cH_2, \err_{S}(h)=0, \err_{S'}(f_{h,S})\geq \frac{\epsilon}{2})\\
        \leq & \Pr(\exists h\in \cH_2, \hat M_{S_\cX\cap R_2}(h)=0, \hat M_{S'_\cX\cap R_2}(h)\geq \frac{\epsilon m}{4})\\
        \leq &(\frac{2em}{d})^{d\log^2 m}\cdot 2^{-\frac{\epsilon m}{4}} = 2^{-\frac{\epsilon m}{4}+ d\log^2 m \log(2em/d)}
    \end{align*}
    \end{itemize}
    Combining the results for $\cH_1$ and $\cH_2$, we have
    \begin{align*}
        \Pr(B_{S,S'}) \leq  2^{-\frac{\epsilon m}{8}} + 2^{-\frac{\epsilon m}{4}+ d\log^2 m \log(2em/d)}\leq  \frac{\delta}{2},
    \end{align*}
    when $m\geq \frac{8}{\epsilon}(d\log^2 m\log \frac{2em}{d} + \log \frac 4 \delta + e) + 4$.
\end{proof}
\begin{proof}[Proof of the upper bound of the 1-inclusion-graph predictor]
    The algorithm is similar to the 1-inclusion-graph predictor in Theorem~\ref{thm:inv-opt}.
    For any $t\in \NN$ and $S=\{(x_1,y_1),\ldots, (x_t,y_t)\}$, let $X_S$ be the set of different elements in $S_{\cX}$ and $\cH'(X_S):= \{h_{|X_S}|\forall x',x\in X_S, x'\in \cG x \text{ implies } h(x')=h(x)\}$.
    Here $\cH'(X_S)$ is different from $\cH(X_S)$ defined in Theorem~\ref{thm:inv-opt} in the sense that every hypothesis in $\cH'(X_S)$ is not $(\cG,X_S)$-invariant but only predict the observed examples in the same orbit in the same way.
    Note that $h^*_{|X_S}$ is in $\cH'(X_S)$ if $S$ is realized by $h^*$.
    Let $Q_{\cH'(X_S),X_S}$ be the function guaranteed by Eq~\eqref{eq:1-inclusion} for the instance space $X_S$ and hypothesis class $\cH'(X_S)$.
    Given a set $S$ of $t$ i.i.d. samples, we let $\cA(S)$ be defined as $\cA(S)$ be defined as $\cA(S, x) = Q_{\cH'(X_{S}\cup \{x\}), X_{S}\cup \{x\}}(S, x)$ if $\cH'(X_{S}\cup \{x\})$ is nonempty and predicting arbitrarily if it is empty.
    Following the analogous proof of Theorem~\ref{thm:inv-opt}, we have
    \begin{align*}
        \EEs{S\sim \cD^t}{\err(\cA(S))} =& \EEs{(x_i,y_i)_{i\in[t+1]}\sim \cD^{t+1}}{\ind{\A(\{x_{i},y_{i}\}_{i\in [t]},x_{t+1})\neq y_{t+1}}} \\
        =& \frac{1}{(t+1)!} \sum_{\sigma\in \text{Sym}(t+1) } \EE{\ind{\A(\{x_{\sigma(i)},y_{\sigma(i)}\}_{i\in [t]},x_{\sigma(t+1)})\neq y_{\sigma(t+1)}}}\\
        =& \EE{\frac{1}{(t+1)!} \sum_{\sigma\in \text{Sym}(t+1) } \ind{\A(\{x_{\sigma(i)},y_{\sigma(i)}\}_{i\in [t]},x_{\sigma(t+1)})\neq y_{\sigma(t+1)}}}\\
        \leq &  \frac{\EE{\vcd(\cH'(\{x_i|i\in [t+1]\}))}}{t+1}\leq \frac{\vcao(\cH,\cG)}{t+1}\,.
    \end{align*}
Again, we use the same method as that in Theorem~\ref{thm:inv-opt} to convert this algorithm, guaranteeing the expected error upper bounded by $\vcao(\cH,\cG)$, into an algorithm with high probability $1-\delta$.
Specifically, the algorithm runs $\cA$ for $\ceil{\log(2/\delta)}$ times, each time using a new sample of size $\ceil{4\vcao(\cH,\cG)/\epsilon}$.
Then the algorithm selects the hypothesis from the outputs with the minimal error on a new sample of size $\ceil{32/\epsilon(\ln(2/\delta)+\ln(\ceil{\log(2/\delta)}+1))}$.
\end{proof}
\section{Proof of Theorem~\ref{thm:ag-opt}}\label{app:ag-opt}
To prove the theorem, we will use a generalization bound for agnostic compression scheme by~\cite*{graepel2005pac}.
\begin{lemma}[Agnostic compression generalization bound]\label{lmm:ag-compression}
    There exists a finite numerical constant $c>0$ such that, for any compression scheme $(\kappa,\rho)$, for any $n\in \NN$ and $\delta\in (0,1)$, for any distribution $\cD$ on $\cX\times \cY$, for $S\sim \cD^n$,
    letting $B(S,\delta):=\frac{1}{n}(\abs{\kappa(S)}\log(n)+\log(1/\delta))$,
    with probability at least $1-\delta$, then
    \begin{equation*}
        \abs{\err_\cD(\rho(\kappa(S)))-\err_S(\rho(\kappa(S)))}\leq c\sqrt{\err_S(\rho(\kappa(S)))B(S,\delta)} + c B(S,\delta)\,.
    \end{equation*}
\end{lemma}

\begin{proof}[Proof of the lower bound]
    For the lower bound, our construction follows \cite{ben2014sample}.
    For any $\cX,\cH$, let $A_1,\ldots,A_d$ be subsets of $\cX$ such that the orbits of every two different elements $x,x'\in \cup_{i\in[d]} A_i$ are disjoint, i.e., $\forall x,x'\in \cup_{i\in[d]} A_i, \cG x\cap \cG x' =\emptyset$.
    We say $\cH$ set-shatters $A_1,\ldots,A_d$ if for every binary vector $\by\in \cY^d$, there exists some $h_\by\in \cH$ such that for all $i\in [d]$ and $x\in \cX$, if $x\in A_i$ then $h_\by(x) = y_i$.
    Then by following Theorem~7 of \cite{ben2014sample}, if $\cH$ set-shatters $A_1,\ldots,A_d$ for some infinite subsets $A_1,\ldots,A_d$ of $\cX$, the standard agnostic PAC sample complexity of learning $\cH$ under deterministic labels for instance space being $\cup_{i\in [d]} A_i$ is lower bounded by $\frac{d}{\epsilon^2}$ for all $\delta < 1/32$.
    Since any data distribution $\cD$ with $\cD_\cX(\cup_{i\in [d]} A_i)=1$ is $\cG$-invariant, the above lower bound also lower bounds $\cM_\ag(\epsilon,\delta;\cH,\cG)$ for all $\delta < 1/32$.
    Let $h(x) = 0$ for all $x\notin \cup_{i\in[k]} A_i$ and for all $h\in \cH$, $\vcao(\cH,\cG) = d$.
    Thus, $\cM_\ag(\epsilon,1/64;\cH,\cG) = \Omega(\frac{\vcao(\cH,\cG)}{\epsilon^2})$.
    The construction above works for any $d>0$.
\end{proof}
\begin{proof}[Proof of the upper bound of ERM-INV]
    Let $d = \vcao(\cH,\cG)$.
    Consider two sets $S$ and $S'$ of $m$ i.i.d. samples drawn from the data distribution $\cD$ each.
    We denote $A_{S}$ the event of $\{\exists h\in \cH, \err_\cD(f_{h,S})\geq \err_{S}(h)+\epsilon\}$ and $B_{S,S'}=\{\exists h\in \cH, \err_{S'}(f_{h,S})\geq \err_{S}(h)+\frac{\epsilon}{2}\}$.
    By Hoeffding bound, we have $\Pr(B_{S,S'})\geq \Pr(A_S)\cdot \Pr(B_{S,S'}|A_S)\geq \frac{1}{2}\Pr(A_S)$ when $m\geq \frac{2}{\epsilon^2}$.
    The sampling process of $S$ and $S'$ is equivalent to drawing $2m$ i.i.d. samples and then randomly partitioning into $S$ and $S'$ of $m$ each.
    For any fixed $S''$, let us divide $S''$ into two categories in terms of the number of examples in each orbit.
    Let $R_1 = \{x|\abs{\cG x \cap S''_{\cX}}\geq \log^2 m\}$ and $R_2 = S''_\cX\setminus R_1$.
    For any multi-subset $X\subset S''_\cX$ and hypothesis $h$, we denote by $\hat M_{X}(h)$ the number of instances in $X$ misclassified by $h$.
    Let $\cF_{S''}(\cH) := \{f_{h,T}| h\in \cH, T\subset S'', \abs{T}=m\}$. 
    Then we have 
    \begin{align*}
        &\Pr(B_{S,S'}) \\
        = &\Pr(\exists h\in \cH, \err_{S'}(f_{h,S})\geq \err_{S}(h)+\frac{\epsilon}{2})\\
        \leq & \Pr(\exists h\in \cH, (\hat M_{S'_\cX \cap R_1}(f_{h,S})\geq \hat M_{S_\cX \cap R_1}(h)+\frac{\epsilon m}{4}) \vee (\hat M_{S'_\cX \cap R_2}(f_{h,S})\geq \hat M_{S_\cX \cap R_2}(h)+\frac{\epsilon m}{4}))\\
        \leq & \underbrace{\Pr(\exists h\in \cH, \hat M_{S'_\cX \cap R_1}(f_{h,S})\geq \hat M_{S_\cX \cap R_1}(h)+\frac{\epsilon m}{4})}_{(a)}\\
        &+ \underbrace{\Pr(\exists h\in \cH,\hat M_{S'_\cX \cap R_2}(f_{h,S})\geq \hat M_{S_\cX \cap R_2}(h)+\frac{\epsilon m}{4})}_{(b)}\,.
    \end{align*}
    For the first term, since $\err_{\cG S}(f_{h,S})=0$ according to the definition of $f_{h,S}$,
    \begin{align*}
        (a) \leq &\Pr(\exists f\in \cF_{S''}(\cH), \hat M_{S'_\cX \cap R_1}(f)\geq \frac{\epsilon m}{4}, \err_{\cG S}(f)=0)\,.
    \end{align*}
    Again, similar to case~1 in Theorem~\ref{thm:inv-da-ub} (also the first upper bound in Theorem~\ref{thm:re-opt}), 
    for any $f\in \cF_{S''}(\cH)$, 
    we let $X(f)\subset R_1$ denote a minimal set of examples in $R_1$ (breaking ties arbitrarily but in a fixed way) such that $f$ misclassify $X(f)$ and $\abs{(\cG X(f))\cap S''_\cX}\geq \frac{\epsilon m}{4}$ where $\cG X(f) = \{g x |x\in X(f), g\in \cG\}$ is the set of all examples lying in the orbits generated from $X(f)$. 
    Let $K(f) = \cG X(f)$ and $\cK = \{K(f)|f\in \cF_{S''}(\cH)\}$.
    Notice that each example in $X(f)$ must belong to different orbits, otherwise it is not minimal. 
    Besides, each orbit contains at least $\log^2 m$ examples.
    Hence, $\abs{X(f)}\leq \frac{\epsilon m}{4 \log^2 m}$. 
    Since there are at most $\frac{2m}{\log^2 m}$ orbits generated from $R_1$, we have $\abs{\cK} \leq \sum_{i=1}^{\frac{\epsilon m}{4 \log^2 m}} {\frac{2m}{\log^2 m}\choose i} \leq \left(\frac{8e}{\epsilon}\right)^{\frac{\epsilon m}{4 \log^2 m}}$.
    Since $f$ misclassify $X(f)$, all examples in $K(f)$ must go to $S'$ to guarantee $\err_{\cG S}(f)=0$ and thus, 
    we have
    \begin{align*}
        &\Pr(\exists f\in \cF_{S''}(\cH),\hat M_{S'_\cX \cap R_1}(f)\geq \frac{\epsilon m}{4}, \err_{\cG S}(f)=0)\\
        \leq & \Pr(\exists f\in \cF_{S''}(\cH), K(f)\cap S_\cX=\emptyset)=\Pr(\exists K\in \cK, K\cap S_\cX=\emptyset)\\
        \leq &\sum_{K\in \cK} 2^{-\frac{\epsilon m}{4}} \leq \left(\frac{8e}{\epsilon}\right)^{\frac{\epsilon m}{4 \log^2 m}}\cdot 2^{-\frac{\epsilon m}{4}} = 2^{-\frac{\epsilon m}{4}(1-\frac{\log(8 e/\epsilon)}{\log^2 m})}\leq 2^{-\frac{\epsilon m}{8}}\,,
    \end{align*}
    when $m\geq \frac{8e}{\epsilon}$.
    For the second term, since $\hat M_{S'_\cX \cap R_2}(h)\geq \hat M_{S'_\cX \cap R_2}(f_{h,S})$, then we have
    \begin{align}
        (b) \leq & \Pr(\exists h\in \cH, \hat M_{S'_\cX \cap R_2}(h)\geq \hat M_{S_\cX \cap R_2}(h)+\frac{\epsilon m}{4})\nonumber\\
        \leq & (\frac{2em}{d})^{d\log^2 m}\cdot e^{-2m'(\frac{\epsilon m}{8m'})^2} \label{eq:apply-hoeffding}\\
        \leq & e^{-\frac{\epsilon^2m}{32}+ d\log^2 m \ln(2em/d)}\nonumber\,,
    \end{align}
    where Eq~\eqref{eq:apply-hoeffding} adopts Hoeffding bound and Sauer's lemma (the number points in $R_2$ that can be shattered by $\cH$ is at most $d\log^2 m$; otherwise $\vcao(\cH,\cG)> d$) and $m' = \max_{h\in \cH} M_{ R_2}(h)$.
    By setting $m\geq \frac{32}{\epsilon^2}(d\log^2 m \ln(2em/d) + \ln(8/\delta)+1) + \frac{8}{\epsilon}\log(8/\delta)$, we have $\Pr(A_S)\leq \delta/2$.
    By Hoeffding bound, with probability at least $1-\delta/2$, $\err_S(h^*)\leq \err_\cD(h^*)+\frac{\epsilon}{2}$.
    By a union bound, we have that with probability at least $1-\delta$, $\err_\cD(f_{ 
    h,S})\leq \err_S(h) + \epsilon \leq \err_S(h^*) + \epsilon \leq \err_\cD(h^*)+ 3\epsilon/2$.
\end{proof}

\begin{proof}[Proof of the upper bound of the 1-inclusion-graph predictor]
    Here we provide another algorithm based on the technique of reduction-to-realizable of \cite{david2016supervised} and the 1-inclusion-graph predictor in the relaxed realizable setting.
    Following the argument in Theorem~\ref{thm:re-opt}, given a sample $S=\{(x_1,y_1),\ldots, (x_n,y_n)\}$, let $X_S$ be the set of different elements in $S_{\cX}$ and define $\cH'(X_S):= \{h_{|X_S}|\forall x',x\in X_S, x'\in \cG x \text{ implies } h(x')=h(x)\}$.
    Then we define an algorithm $\cA_0$ as 
    \begin{align}
        \cA_0(S,x) = 
        \begin{cases}
            Q_{\cH'(X_{S}\cup \{x\}), X_{S}\cup \{x\}}(S, x)
            &\text{if } \cH'(X_{S}\cup \{x\})\neq \emptyset\,,\\
            0 &\text{ o.w.}
        \end{cases}\label{eq:def-of-A}
    \end{align}
    where $Q_{\cH'(X_{S}\cup \{x\}), X_{S}\cup \{x\}}$ is the function guaranteed by Eq~\eqref{eq:1-inclusion} for hypothesis class $\cH'(X_{S}\cup \{x\})$ and instance space $ X_{S}\cup \{x\}$.
    In the relaxed realizable setting, if $S\sim \cD^n$ and $x\sim \cD_\cX$, $\cH'(X_{S}\cup \{x\})\neq \emptyset$ a.s. as it contains $h^*_{|X_{S}\cup \{x\}}$. 
    While in the agnostic setting, it is not the case.
    Even if $S\sim \cD^n$ and $x\sim \cD_\cX$, $\cH'(X_{S}\cup \{x\})$ could be empty as there might exist an instance $x'\in X_S \cap \cG x$ such that no hypothesis in $\cH$ labeling them in the same way.
    
    Let $d = \vcao(\cH,\cG)$.
    If a sample $\{(x_1,y_1),\ldots, (x_{n+1},y_{n+1})\}$ is realizable by $\cH$, then by Lemma~\ref{lmm:1-inclusion}, we have
    \begin{align}
        \frac{1}{(n+1)!}\sum_{\sigma\in \Sym(n+1)}\ind{\cA_0(\{x_{\sigma(i)},y_{\sigma(i)}\}_{i\in [n]},x_{\sigma(n+1)})\neq y_{\sigma(n+1)}} \leq \frac{d}{n+1}\,.\label{eq:apply-ag-1-inclusion}
    \end{align}
    
    Now we use $\cA_0$ as a weak learner to construct a compression scheme by following the construction \cite{david2016supervised}.
    Given a training set $S_\trn=\{(x_1,y_1),\ldots,(x_m,y_m)\}$, let $R$ denote the largest submulitset of $S_\trn$ that is realizable w.r.t. $\cH$.
    If $\abs{R} = 0$, then define $\hat h$ as the all-0 function $\hat h(x)=0$.
    Otherwise, if $\abs{R}>0$, 
    for any distribution $\cP$ on $R$,
    by Eq~\eqref{eq:apply-ag-1-inclusion}, we have that
    \begin{align*}
        &\EEs{S\sim \cP^{3d}}{\err_\cP(\cA_0(S))}\\
        =&\EEs{(x_i,y_i)_{i\in[3d+1]}\sim  \cP^{3d}}{\frac{1}{(3d+1)!}\sum_{\sigma\in \Sym(3d+1)}\ind{\cA_0(\{x_{\sigma(i)},y_{\sigma(i)}\}_{i\in [3d]},x_{\sigma(3d+1)})\neq y_{\sigma(3d+1)}} }\\
        \leq& \frac{1}{3}\,.
    \end{align*}
    Hence, there exists $S_\cP\in R^{3d}$ with $\err_\cP(\cA_0(S_\cP))\leq 1/3$.
    Thus, the algorithm $\cA_0$ can serve as a weak learner.
    By Lemma~\ref{lmm:boosting}, for $T=\ceil{c\log \abs{R}}$ (for a numerical constant $c>0$), there exist $S_1,\ldots, S_T\in R^{3d}$ such that, letting $\hat h (\cdot)= \Majority(\cA_0(S_1)(\cdot),\ldots,\cA_0(S_T)(\cdot))$, we have $\err_{R}(\hat h)=0$.
    Thus, $\err_{S_\trn}(\hat h)\leq \inf_{h\in \cH} \err_{S_\trn}(h)$.
    Here $\hat h$ is the output of the compression scheme that selects $\kappa(S_\trn) = (S_1,\ldots, S_T)$ and $\rho(\kappa(S_\trn))=\hat h$.
    By Lemma~\ref{lmm:ag-compression}, with probability at least $1-\delta/2$,
    \begin{align*}
        \err_\cD(\hat h) \leq \err_{S_\trn}(\hat h) + c'\sqrt{\err_{S_\trn}(\hat h)B(S_\trn,\delta/2)} + c' B(S_\trn,\delta/2)\,,
    \end{align*}
    for a numerical constant $c'>0$.
    Following the same argument of \cite{david2016supervised} (Lemma~3.2), we have that 
    \begin{equation*}
        \Pr_{S_\trn\sim \cD^m}(\err_S(\hat h)\geq \inf_{h\in \cH} \err_{\cD}(h) + \sqrt{\log(2/\delta)/m})\leq \delta/2\,.
    \end{equation*}
    By taking a union bound, then we have that with probability at least $1-\delta$,
    \begin{align*}
        \err_\cD(\hat h) \leq &\inf_{h\in \cH} \err_{\cD}(h) + c'\sqrt{(\inf_{h\in \cH} \err_{\cD}(h) + \sqrt{\log(2/\delta)/m})\frac{1}{m}(3d\ceil{c\log\abs{R}}\log m + \log1/\delta)} \\
        &+\frac{c'}{m}(3d\ceil{c\log\abs{R}}\log m + \log(1/\delta))\,.
    \end{align*}
    Hence,
    \begin{equation*}
        \cM_\ag(\epsilon,\delta;\cH,\cG) = O\left(\frac{d}{\epsilon^2}\log^2\left(\frac{d}{\epsilon}\right)+\frac{1}{\epsilon^2}\log(1/\delta)\right)\,.
    \end{equation*}
\end{proof}
\section{Adaptive algorithms}
For any hypothesis $h\in \cH$, we say $h$ is $(1-\eta)$-invariant over the distribution $\cD_{\cX}$ for some $\eta\in [0,1]$ if $\PPs{x\sim \cD_{\cX}}{\exists x'\in \cG x, h(x')\neq h(x)}= \eta$.
We call $\eta(h)=\eta$ the invariance parameter of $h$ with respect to $\cD_{\cX}$.
In the relaxed realizable setting, the problem degenerates into the invariantly realizable setting when $\eta(h^*)=0$.
This implies that we can benefit from transformation invariances more when $\eta(h^*)$ is smaller.
The case is similar in the agnostic setting.
In this section, we discuss adaptive learning algorithms for different levels of invariance of the target function.

\subsection{An adaptive algorithm in the relaxed realizable setting}\label{app:unified-re}
There might exist more than one hypotheses in $\cH$ with zero error and we let the target function $h^*$ be the one with the smallest invariance parameter (breaking ties arbitrarily). 
For any multiset $X$ in $\cX$, for any hypothesis $h$, denote $h_{|X}$ the restriction of $h$ on the set of different elements in $X$.
Then we introduce a distribution-dependent dimension as follows.

\begin{definition}[approximate $(1-\eta)$-invariant VC dimension]\label{def:re-appx-dim}
    For any $\eta\in [0,1]$ and finite multi-subset $X\subset \cX$,
    let $\cH^{\eta}(X) := \{h_{|X}|\frac{1}{\abs{X }}\sum_{x\in X}{\ind{\exists x'\in \cG x, h(x')\neq h(x)}}\leq \eta \wedge \forall x,x'\in \cG x \cap X, h(x')=h(x)\}$.
    For any $m\in \NN$, marginal distribution $\cD_\cX$ and target function $h^*$, the approximate $(1-\eta)$-invariant VC dimension is defined as
    \[\vco^{\eta}(m,h^*,\cH,\cG,\cD_\cX):= \EEs{X\sim \cD_{\cX}^{m+1}}{\vcd(\cH^{\eta}(X))|h^*_{|X}\in \cH^{\eta}(X)}\,,\]
    when $\Pr(h^*_{|X}\in \cH^{\eta}(X))>0$ and $\vco^{\eta}(m,h^*,\cH,\cG,\cD_\cX) =0$ when $\Pr(h^*_{|X}\in \cH^{\eta}(X))=0$.
    By taking supremum over $m$, 
    \[\vco^{\eta}(h^*,\cH,\cG,\cD_\cX) = \sup_{m\in \NN}\vco^{\eta}(m,h^*,\cH,\cG,\cD_\cX)\,.\]
\end{definition}
In this definition, all hypotheses in $\cH^{\eta}(X)$ need to satisfy two constraints: a) the empirical invariance parameter is less than or equal to $\eta$, and b) the prediction over instances in $X$ is invariant over orbits.
Here the second constraint arises due to the fact that $h^*$ satisfies this constraint.
Note that $\vco^{\eta}(h^*,\cH,\cG,\cD_\cX)$ is monotonic increasing in $\eta$.
For all $\cD_\cX, h^*$, we have $\vco^{0}(h^*,\cH,\cG,\cD_\cX)\leq \vco(\cH,\cG)$ and $\vco^{\eta}(h^*,\cH,\cG,\cD_\cX)\leq \vcao(\cH,\cG)$ for all $\eta\in [0,1]$.
We will use the approximate $(1-\eta)$-invariant VC dimension to characterize the sample complexity dependent on $\eta$.
Ideally, it is heuristic to adopt a notion like $\vcao(\{h\in \cH|\eta(h) = \eta(h^*)\})$.
But this is impossible to achieve as we cannot obtain an accurate estimate of $\eta(h^*)$ via finite data points.

\begin{proposition}\label{prop:re-eta-known-ub}
    If it is known that $\eta(h^*)\leq \eta$ for some $\eta\in [0,1]$, for any training sample size $m\in \NN$, for any $\Delta \geq \sqrt{\frac{\ln (n+1)}{2(n+1)}}$ with $n=\Theta(\frac{m}{\log(1/\delta)})$, there is an algorithm achieving error $O(\frac{\vco^{\eta+\Delta}(h^*,\cH,\cG,\cD_\cX)\log(1/\delta)}{m})$ with probability at least $1-\delta$.
\end{proposition}
Given a sample $S=\{(x_1,y_1),\ldots,(x_n,y_n)\}$ and a test instance $x$, 
let $X_S$ denote the set of different elements in $S_\cX$ and $X=\{x_1,\ldots, x_{n},x\}$ the multiset of all unlabeled instances from both the training set and the test instance.
The algorithm $\cA$ is defined by 
\begin{equation*}
    \cA(S, x) = Q_{\cH^{\eta+\Delta}(X),X_S\cup\{x\}}(S,x)\,,
\end{equation*}
where $Q_{\cH^{\eta+\Delta}(X),X_S\cup\{x\}}$ is the function guaranteeed by Eq~\eqref{eq:1-inclusion} for the instance space $X_S\cup\{x\}$ and the concept class $\cH^{\eta+\Delta}(X)$.
Similar to Theorem~\ref{thm:inv-opt}, algorithm $\cA$ can achieve expected error $\frac{\vco^{\eta+\Delta}(h^*,\cH,\cG,\cD_\cX)+1}{n+1}$.
Then by following the same confidence boosting argument, we run $\cA$ for $\ceil{\log (2/\delta)}$ times on independent new samples and select the output hypothesis with minimum error a new sample.
The details of the algorithm and the proof of Proposition~\ref{prop:re-eta-known-ub} are deferred to Appendix~\ref{app:re-eta-known-ub}. 

In the more general case where $\eta$ is unknown, we build an algorithm based on the algorithm for known $\eta$ above.
Denote $\cA_{\eta,\Delta}$ the algorithm satisfying the guarantee in Proposition~\ref{prop:re-eta-known-ub} with probability $1-\delta/2$ for hyperparameters $\eta$ and $\Delta$.
Then we divide $[0,1]$ into uniform intervals and then search for the interval where $\eta(h^*)$ lies in.
The detailed algorithm is provided in Algorithm~\ref{alg:unif-re}.

\begin{algorithm}[t]\caption{An adaptive algorithm in the relaxed realizable setting}\label{alg:unif-re}
    \begin{algorithmic}[1]
    \STATE Input: a labeled sample $S_\trn$ of size $m$, $m_1\in \NN$, $\Delta\in [0,1]$
    \STATE Randomly partition $S$ into $S_1$ with $\abs{S_1}=m_1$ and $S_2$ with $\abs{S_2}=m-m_1$
    \FOR{$i=0,1,\ldots,\ceil{1/(2\Delta)}$}
    \STATE Let $h_i = \cA_{(2i-1)\Delta,\Delta}(S_1)$
    \STATE Return $\hat h = \argmin_{h\in \{h_i|i\in \{0,\ldots,\ceil{1/(2\Delta)}\}\}}\err_{S_2}(h)$
    \ENDFOR
    \end{algorithmic}
  \end{algorithm}

\begin{theorem}\label{thm:re-eta-unknown-ub}
    Set $\Delta = \sqrt{\frac{\ln (n+1)}{2(n+1)}}$ for $n=\Theta(\frac{m}{\log (m) \log(1/\delta))})$ and $\abs{S_1} = \Theta(\frac{m}{\log m})$.
    Let $i^*\geq 0$ be the smallest integer $i$ such that $\max((2i-1)\Delta,0)\geq \eta(h^*)$. 
    Then Algorithm~\ref{alg:unif-re} achieves error $O(\frac{\vco^{2i^* \Delta}(h^*,\cH,\cG,\cD_\cX)\log(1/\delta)\log(m)}{m})$ with probability at least $1-\delta$.
\end{theorem}
The proof is deferred to Appendix~\ref{app:re-eta-unknown-ub}.
The algorithm above perform close to optimally in both the invariant realizable setting and the relaxed realizable setting.
Intuitively, the algorithm above outperforms PAC-optimal algorithms in Theorem~\ref{thm:re-opt} when $\eta(h^*)$ is small.
Below is an example showing the advantage of this adaptive algorithm in the extreme case of $\eta(h^*) = 0$.
\begin{example}
Consider the construction of $\cH_d,\cG_d, h^*, \cD$ in the proof of Theorem~\ref{thm:inv-da-lb}. 
In this example, we have $i^*=0$ and $\cH^0(X)=\{\bOne\}$ for any $X$.
Thus, $\vco^0(h^*,\cH,\cG,\cD_\cX) = 0$.
Hence, the algorithm above only requires sample complexity $O(1)$ to output a zero-error predictor.
However, algorithms only caring the worst case upper bound may require much more samples.
For example, ERM-INV predicts exactly the same as standard ERM in this example and thus, it requires $\Omega(\frac{\vcao(\cH,\cG)}{\epsilon})$ samples to achieve $\epsilon$ error.
\end{example}

\paragraph{Open question:} It is unclear whether this adaptive algorithm is optimal and whether the approximate $(1-\eta)$-invariant VC dimension is the best $\eta$-dependent measure to characterize the sample complexity.

\subsection{An adaptive algorithm in the agnostic setting}\label{app:unified-ag}
In the agnostic setting, we assume that there exists an optimal hypothesis $h^*\in \cH$ such that $\err(h^*) = \inf_{h\in \cH}\err(h)$.
Then similar to the realizable setting, we can design algorithms that adapt to $\eta(h^*)$.
However, it is more challenging to design an adaptive algorithm in the agnostic setting than in the relaxed realizable setting.
One of the most direct ideas is to combine agnostic compression scheme with the adaptive algorithm in the realizable setting.
One possible way of combination is finding the largest realizable subset of the data and applying the adaptive algorithm in the relaxed realizable setting.
However, this does not work since the realizable subset is not i.i.d. and the empirical invariance parameter calculated based on this subset is biased.
Another possible way of combination is calculating the empirical invariance parameter over the whole data set, reducing the hypothesis class based on this empirical value and then run the compression scheme in Theorem~\ref{thm:ag-opt} based on this reduced hypothesis class.
This does not work either because the predictor depends on the whole data set now and the compression size is too large.
Hence, there is a significant barrier that has arisen as a result of estimating the invariance parameter while obtaining low error at the same time.

To get around this obstacle, we provide an approach of using two independent data sets.
Specifically, we partition the hypothesis class into subclasses with different empirical invariance parameters based on a data set first.
Notice that in this step, we only need an unlabeled data set.
Then we run the compression scheme in Theorem~\ref{thm:ag-opt} for each subclass and return the one with the small validation error.
The detailed algorithm is presented in Algorithm~\ref{alg:unif-ag}.

\begin{algorithm}[t]\caption{An adaptive algorithm in the agnostic case}\label{alg:unif-ag}
  \begin{algorithmic}[1]
  \STATE Input: a unlabeled data set $U$ of size $u$ and a labeled data set $S$ of size $m$,$m_1\in \NN$, $\Delta>0$
  \STATE Randomly divide $S$ into $S_1$ of size $m_1$ and $S_2$ of size $m-m_1$
  \STATE Use $U$ to partition $\cH$ into $\hat \cH_1,\ldots,\hat \cH_K$ with $\hat \cH_i = \{h\in \cH| \frac{1}{u}\sum_{x\in U} \ind{\exists x'\in \cG x, h(x')\neq h(x)} \in (2(i-1)\Delta, 2i\Delta]\}$ and $K = \ceil{\frac{1}{2\Delta}}$
  \FOR{$i=0,1,2,\ldots,K$}
  \STATE Run the algorithm $\cA$ in Theorem~\ref{thm:ag-opt} over $S_1$ for hypothesis class $\hat \cH_i$ and output $h_i$
  \STATE Return $\hat h = \argmin_{h\in \{h_i|i\in \{0,\ldots,K\}\}} \err_{S_2}(h)$
  \ENDFOR
  \end{algorithmic}
\end{algorithm}

\begin{theorem}\label{thm:ag-eta}
    For each $h\in \cH$, the invariance indicator function of $h$ is a mapping $\iota_h: \cX\mapsto \{0,1\}$ such that $\iota_h(x) = \ind{\exists x'\in \cG x, h(x')\neq h(x)}$.
    Denote $\cI = \{\iota_h|h\in \cH\}$ the set of invariance indicator functions for all $h\in \cH$.
    Then for any $\Delta \in (0,1]$, Algorithm~\ref{alg:unif-ag} can achieve $\err(\hat h) \leq \err(h^*) + O\left(\sqrt{\frac{\vcao(\cH^*)\log^2 m + \log(1/\delta) + \log(1/\Delta)}{m}}\right)$, where $\cH^* = \{h| \eta(h)\in (\eta(h^*) -2\Delta -2 \Delta',\eta(h^*) +2\Delta + 2\Delta')\}$ with $\Delta' = \Theta(\sqrt{\frac{\vcd(\cI)\log(u)+\log(1/\delta)}{u}})$.
\end{theorem}
The detailed proof of Theorem~\ref{thm:ag-eta} is deferred to Appendix~\ref{app:ag-eta}.
The upper bound in Theorem~\ref{thm:ag-eta} depends on $\vcd(\cI)$, which can be arbitrarily larger than $\vcd(\cH)$.
For example, for any $d>0$, let $\cX = \{0,1\}^d\times [d]$.
For each $\bb\in \{0,1\}^d$, define a hypothesis $h_\bb$ by letting $h_\bb((\bb,i)) = b_i$ and $h_\bb(x) = 0$ for all other $x\in \cX$.
Let the hypothesis class $\cH = \{h_\bb|\bb\in \{0,1\}^d\}$.
Then it is direct to check that $\vcd(\cH) = 1$ but $\vcd(\cI) = d$.
It is unclear how to design an adaptive algorithm with theoretical guarantee independent of $\vcd(\cI)$.
The $\eta$-dependent dimension we adopt in the agnostic setting is different from that in the relaxed realizable setting.
It is also unclear what is the best way to characterize the dependence on $\eta$ in the agnostic setting.

\section{Proof of Proposition~\ref{prop:re-eta-known-ub}}\label{app:re-eta-known-ub}
\begin{proof}
    The proof follows that of Theorem~\ref{thm:inv-opt}.
    Given a training sample of size $m$, let $n$ be the largest integer such that $m \geq n\ceil{\log(2/\delta)} + \ceil{8(n+1)(\ln(2/\delta)+\ln(\ceil{\log(2/\delta)}+1))}$.
    For any sample $S=\{(x_1,y_1),\ldots,(x_n,y_n)\}$ and a test instance $x_{n+1}=x$,
    if $h^*$ is in $\cH^{\eta+\Delta}(X)$, by Lemma~\ref{lmm:1-inclusion}, we have that
    \begin{align*}
        \frac{1}{(n+1)!}\sum_{\sigma\in \text{Sym}(n+1) } \ind{\cA(\{x_{\sigma(i)},y_{\sigma(i)}\}_{i\in [n]},x_{\sigma(n+1)})\neq y_{\sigma(n+1)}}\leq \frac{\vcd(\cH^{\eta+\Delta}(X))}{n+1}\,.
    \end{align*}
    Thus,
    \begin{align}
        &\EEs{S\sim \cD^n}{\err(\cA(S))} \nonumber\\
        =& \EEs{(x_i,y_i)_{i\in[n+1]}\sim \cD^{n+1}}{\ind{\A(\{x_{i},y_{i}\}_{i\in [n]},x_{n+1})\neq y_{n+1}}} \nonumber\\
        =& \frac{1}{(n+1)!} \sum_{\sigma\in \text{Sym}(n+1) } \EE{\ind{\A(\{x_{\sigma(i)},y_{\sigma(i)}\}_{i\in [n]},x_{\sigma(n+1)})\neq y_{\sigma(n+1)}}}\nonumber\\
        =& \EE{\frac{1}{(n+1)!} \sum_{\sigma\in \text{Sym}(n+1) } \ind{\A(\{x_{\sigma(i)},y_{\sigma(i)}\}_{i\in [n]},x_{\sigma(n+1)})\neq y_{\sigma(n+1)}}}\nonumber\\
        \leq & \EEc{\frac{1}{(n+1)!} \sum_{\sigma\in \text{Sym}(n+1) } \ind{\A(\{x_{\sigma(i)},y_{\sigma(i)}\}_{i\in [n]},x_{\sigma(n+1)})\neq y_{\sigma(n+1)}}}{h^*_{|X}\in \cH^{\eta+\Delta}(X)} \nonumber\\
        & \cdot\Pr(h^*_{|X}\in \cH^{\eta+\Delta}(X))+ \Pr(h^*_{|X}\notin \cH^{\eta+\Delta}(X))\nonumber\\
        \leq & \frac{\EEs{X\sim \cD_{\cX}^{n+1}}{\vcd(\cH^{\eta+\Delta}(X))|h^*_{|X}\in \cH^{\eta+\Delta}(X)}}{n+1}\Pr(h^*_{|X}\in \cH^{\eta+\Delta}(X))  \nonumber\\
        &+ \Pr(h^*_{|X}\notin \cH^{\eta+\Delta}(X))\nonumber\\
        \leq & \frac{\vco^{\eta+\Delta}(h^*,\cH,\cG,\cD_\cX)}{n+1}\Pr(h^*_{|X}\in \cH^{\eta+\Delta}(X)) + \Pr(h^*_{|X}\notin \cH^{\eta+\Delta}(X))\nonumber\\
        \leq & \frac{\vco^{\eta+\Delta}(h^*,\cH,\cG,\cD_\cX)+1}{n+1}\,,\label{eq:unif-exp-er}
    \end{align}
    where Eq~\eqref{eq:unif-exp-er} holds due to $\Pr(h^*_{|X}\notin \cH^{\eta+\Delta}(X))\leq \Pr(\frac{1}{\abs{X }}\sum_{x\in X}{\ind{\exists x'\in \cG x, h^*(x')\neq h^*(x)}}-\eta(h^*)\geq \Delta)\leq \exp(-2(n+1)\Delta^2)$ by Hoeffding bound.

    Then we again follow the classic technique to boost the confidence.
    The algorithm runs $\cA$ for $\ceil{\log(2/\delta)}$ times, each time using a new sample $S_i$ of size $n$ for $i=1,\ldots,\ceil{\log(2/\delta)}$.
    Let $h_i = \cA(S_i)$ and then selects the hypothesis $\hat h$ from $\{h_i|i\in [\ceil{\log(2/\delta)}]\}$ with the minimal error on a new sample $S_0$ of size $t=\ceil{8(n+1)(\ln(2/\delta)+\ln(\ceil{\log(2/\delta)}+1))}$.

    Denote $\epsilon = \frac{\vco^{\eta+\Delta}(h^*,\cH,\cG,\cD_\cX)+1}{n+1}$.
    For each $i$, by Eq~\eqref{eq:unif-exp-er}, we have $\EE{\err(h_i)}\leq \epsilon$. By Markov's inequality, with probability at least $\frac{1}{2}$, $\err(h_i)\leq 2\epsilon$.
    Since $h_i$ are independent, we have that with probability at least $1-\frac{\delta}{2}$, at least one of $\{h_i|i\in [\ceil{\log(2/\delta)}]\}$ has error smaller than $2\epsilon$.
    Then by Chernoff bound, for each $i$, on the event $\err(h_i)\leq 2\epsilon$,
    \begin{equation*}
        \Pr(\err_{S_0}(h_i)> 3\epsilon|h_i)< e^{-\frac{t\epsilon}{6}}\,.
    \end{equation*}
    Also, on the event $\err(h_i)> 4\epsilon$,
    \begin{equation*}
        \Pr(\err_{S_0}(h_i)\leq 3\epsilon|h_i)\leq e^{-\frac{t\epsilon}{8}}\,.
    \end{equation*}
    Thus, by the law of total probability and a union bound, with probability at least $1-\frac{\delta}{2}$, if any $i$ has $\err(h_i)\leq 2\epsilon$, then the returned the hypothesis $\hat h$ has $\err(\hat h)\leq 4\epsilon$.
    By a union bound, the proof is completed.
\end{proof}
\section{Proof of Theorem~\ref{thm:re-eta-unknown-ub}}\label{app:re-eta-unknown-ub}
\begin{proof}
    Given a training sample of size $m$, let $\abs{S_1} = m_1$ and $\abs{S_2}=m_2=m-m_1$.
    The values of $m_1,m_2$ are determined later.
    Let $n$ be the largest integer such that $m_1 \geq n\ceil{\log(2/\delta)} + \ceil{8(n+1)(\ln(2/\delta)+\ln(\ceil{\log(2/\delta)}+1))}$.
    Let $\Delta = \sqrt{\frac{\ln (n+1)}{2(n+1)}}$ and then the number of rounds we run $\cA_{\eta, \Delta}$ as a subroutine is upper bounded by $\ceil{1/(2\Delta)}+1=\ceil{\sqrt{(n+1)/(2\ln(n+1))}}+1\leq m$.
    According to Proposition~\ref{prop:re-eta-known-ub}, there is a numerical constant $c>0$ such that with probability $1-\delta/2$, $\err(h_{i^*})\leq \frac{c \vco^{2i^*\Delta}(h^*,\cH,\cG,\cD_\cX)\ln(2/\delta)}{m_1}$.
    Denote $\epsilon = \frac{c \vco^{2i^*\Delta}(h^*,\cH,\cG,\cD_\cX)\ln(2/\delta)}{m_1}$.
    By Chernoff bound, for each $i \neq i^*$, on the event $\err(h_i)>4\epsilon$,
    \begin{equation*}
        \Pr(\err_{S_2}(h_i)\leq 2\epsilon|h_i)\leq e^{-\frac{m_2\epsilon}{2}}\,.
    \end{equation*}
    Also, on the event $\err(h_{i^*})\leq \epsilon$,
    \begin{equation*}
        \Pr(\err_{S_2}(h_{i^*})> 2\epsilon|h_{i^*})< e^{-\frac{m_2\epsilon}{3}}\,.
    \end{equation*}
    Let $m_1 = \frac{cm}{3\ln m +c}$ and $m_2 = \frac{3m\ln m}{3\ln m +c}$. 
    Then with probability at least $1-m e^{-m_2\epsilon/3}\geq 1-\delta/2$, if $\err(h_{i^*})\leq \epsilon$, then the returned classifier has error smaller than $4\epsilon$.
    By taking a union bound, the proof is completed.
\end{proof}
\section{Proof of Theorem~\ref{thm:ag-eta}}\label{app:ag-eta}
\begin{proof}
    Given a labeled data set $S$ of size $m$, let $\abs{S_1} = m_1$ and $\abs{S_2}=m_2=m-m_1$.
    The values of $m_1, m_2$ are determined later.
    Let $\hat i$ be the $i\in [K]$ such that $h^*\in \hat \cH_i$.
    By Theorem~\ref{thm:ag-opt}, we have that there exists a numerical constant $c>0$ such that with probability at least $1-\delta/3$,
    \[\err(h_{\hat i})\leq \err(h^*) + c\sqrt{\frac{\vcao(\hat \cH_{\hat i})\log^2m_1 + \log(1/\delta)}{m_1}}\,.\]
    Then by Hoeffding bound, for any $\epsilon>0$, for each $i\in K$, with probability at least $1-\delta/(3K)$,
    \[\abs{\err_{S_2}(h_i)-\err(h_i)} \leq \epsilon\,,\]
    when $m_2 \geq \frac{\ln (6K/\delta)}{2{\epsilon}^2}$.
    Then by taking a union bound, with probability at least $1-\delta/3$, for all $i\in K$ we have
    \[\abs{\err_{S_2}(h_i)-\err(h_i)} \leq \epsilon\,.\]
    Hence, with probability at least $1-2\delta/3$,
    \[\err(\hat h) \leq \err(h_{\hat i}) + 2 \epsilon \leq \err(h^*) + c\sqrt{\frac{\vcao(\hat \cH_{\hat i})\log^2m_1  + \log(1/\delta)}{m_1}} + 2\epsilon\,.\]
    By uniform convergence bound, with probability at least $1-\delta/3$,
    \[\abs{\frac{1}{u}\sum_{x\in U} \ind{\exists x'\in \cG x, h(x')\neq h(x)} - \eta(h)} \leq \Delta', \forall h\in \cH\,.\]
    It follows that $\hat \cH_{\hat i} \subset \cH^*$ and thus, $\vcao(\hat \cH_{\hat i}) \leq \vcao(\cH^*)$.
    The proof is completed by letting $\epsilon = \sqrt{\frac{\vcao(\cH^*) + \log(1/\delta)}{m_1}}$ and $m_1 = \frac{2(\vcao(\cH^*) + \log(1/\delta))m}{2(\vcao(\cH^*) + \log(1/\delta)) + \ln(6K/\delta)}$.
\end{proof}
\end{document}